\documentclass{article}
\usepackage[T1]{fontenc}
\usepackage{iclr2027_conference,times}
\iclrfinalcopy 
\renewcommand{\iclrruler}[1]{}
\usepackage{amsmath,amssymb,amsthm,mathtools,mleftright}

\usepackage{booktabs,tabularx}
\usepackage{microtype}
\usepackage{url}
\usepackage{xcolor}
\usepackage{graphicx}
\usepackage{hyperref}
  \hypersetup{hidelinks}
\usepackage{doi}  

\title{Exact-Solution Volume and \\Length Generalization in Transformers}
\author{Yijia Jessica Zhu and David Chiang \\ University of Notre Dame \\ \texttt{\{yzhu26,dchiang\}@nd.edu}}

\usepackage[capitalize,nameinlink]{cleveref}

\newtheorem{theorem}{Theorem}

\newtheorem{lemma}[theorem]{Lemma}
\theoremstyle{definition}
\newtheorem{definition}{Definition}

\newcommand{\FIRST}{\textnormal{\textsc{First}}}
\newcommand{\INDEX}{\textnormal{\textsc{Index}}}
\newcommand{\MAJORITY}{\textnormal{\textsc{Majority}}}
\newcommand{\PARITY}{\textnormal{\textsc{Parity}}}
\newcommand{\Apos}{\mathcal A_{\mathrm{pos}}}
\newcommand{\Acount}{\mathcal A_{\mathrm{gen}}}
\newcommand{\Vol}{V_B}
\newcommand{\R}{\mathbb R}
\newcommand{\1}{\mathbf 1}
\newcommand{\relu}{\operatorname{ReLU}}

\newcommand{\Anorm}{\mathcal A_{\mathrm{norm}}}
\newcommand{\Asin}{\mathcal A_{\mathrm{sin}}}

\begin{document}
\maketitle
\fancyhead{} 

\begin{abstract}
Research on transformer expressivity shows whether a transformer is capable of solving a given task, but gives little indication of whether the solution, if learned, is generalizable to longer input lengths. We study this question through \emph{normalized exact-solution volume (NESV)}: the fraction of a bounded parameter region that achieves an exact solution on every input of length $n$. For fixed-width, single-layer transformers with $\log n$-scaled attention, we establish asymptotic bounds on NESV for four tasks: \FIRST{} ($\Theta(1)$), \MAJORITY{} ($\Theta(1/(n\log n))$), \INDEX{} ($\Theta(1/n^3)$), and \PARITY{} ($0$). These results are consistent with previous empirical results: the faster the exact-solution volume decays with input length, the harder it is to length-generalize on that task. Looking deeper into \INDEX, our volume analysis reveals two error sources that grow with $n$. Consequently, we study a transformer model that would structurally eliminate one of the terms, theoretically improving the NESV bound to $\Theta(n^{-1})$, and empirically achieving 85\% accuracy when tested at $10\times$ the training length, compared with the 60\% accuracy of the original model. We conclude that volume analysis may be a useful approach to identify concrete sources of length sensitivity and thus provide insights into task-specific model refinements.
\end{abstract}

\section{Introduction}

Theoretical work on the expressivity of transformers (cf.~the survey by \citealp{strobl-etal-2024-survey}) can be seen as a prelude to the study of learnability by transformers, since the existence of a transformer and parameters that solve a given problem is a prerequisite for a transformer to learn to solve the problem. But the existence of a single point in parameter space does not tell us much about whether that point is actually learnable. In this paper, following \citet{kover2026geometry}, we ask a more difficult question: given that a problem is solvable by a transformer with some parameters, how large is the set of successful parameters? While \citet{kover2026geometry} arrived at somewhat pessimistic conclusions, we look at single-layer transformers with \emph{$\log n$-scaling} and find that, the volume of the set of successful parameters, as a function of the sequence length $n$, provides a way to quantify the difficulty of length-generalization in various problems. 

The drive for looking beyond expressivity is also visible in experiments on simple sequence tasks. \citet{chiang2022overcoming} construct exact transformers for \FIRST{} (whether the first symbol is a $\texttt{1}$) and \PARITY{} (whether the number of $\texttt{1}$'s is odd), but report that standard training fails on \PARITY, and that \FIRST{} does not necessarily extrapolate to longer inputs without $\log n$-scaling. \Citet{hahn2024sensitive} report that fitted \PARITY{} solutions become much sharper as input length grows than fitted solutions for \FIRST{} and \MAJORITY{} (whether there are more $\texttt{1}$'s than $\texttt{0}$'s). Thus, tasks that are all expressible can behave differently under training and length extrapolation.

Such findings motivate examining the geometry of the parameter sets that implement a task. In related work on loss landscapes, \citet{chiang2023loss} argue that the geometry and volume of low-loss regions can explain part of neural-network generalization without relying entirely on the implicit bias of gradient descent. We study a related exact-solution quantity at each finite input length. For a fixed task and architecture, we restrict the trainable parameters to a bounded region and measure the fraction that correctly solves the task on every input of length $n$. We call this quantity the \emph{normalized exact-solution volume (NESV)}. Its dependence on $n$ describes how the full set of exact solutions changes as the computation is required to handle longer inputs. With uniform sampling over the bounded parameter region, this is the probability of sampling an exact solution. Comparing this baseline with empirical training results can help distinguish the effects of initialization and optimization from the nature of the task and the selected transformer model in terms of learnability.

We consider four binary classification tasks on binary strings. \FIRST\ returns the first input symbol, whereas \INDEX\ returns the symbol at a position specified by one of the input tokens. \MAJORITY\ returns whether more than half of input symbols are 1's, whereas \PARITY\ returns whether the number of $\texttt{1}$'s is odd. We refer to the first pair as \emph{positional} tasks, and the second pair as \emph{counting} tasks. 

We conduct our analyses on the simplest models suited to solve the corresponding problems. In all cases, we only consider single-layer transformers. For the positional tasks, we use a model with positional encoding ($j$, $j^2$) where the Query and Key maps only depend on token positions and the Value map only depends on the tokens themselves (following the form of the construction of \citet{bhattamishra2024separations}). %
For the counting tasks, the positions of the tokens are irrelevant, so we primarily use a model without positional encoding. We also examined a model with sinusoidal positional encoding ($\sin(j)$, $\cos(j)$).

See \cref{tab:summary} for a summary of the asymptotic bounds obtained with respect to input length $n$. 
Under the positional model $\Apos$, the NESV for \INDEX{} is $\Theta(n^{-3})$, while \FIRST\ has constant NESV. The sinusoidal model $\Asin$ has volume $0$ for both. Under the position-free model $\Acount$, the NESV for \MAJORITY{} is $\Theta(1/(n\log n))$. Under the sinusoidal model $\Asin$, the NESV for \MAJORITY{} is instead $\Theta(1/({n^3\log n}))$.
\PARITY{} has a zero NESV whenever $n>h+1$ for both models, $h$ being the width of the transformer. This is a direct consequence of the standard one-dimensional capacity of shallow ReLU networks \citep{arora2018understanding} in the finite-length setting, and is consistent with the broader result that no one-layer, one-head transformer computes \PARITY{} across all lengths \citep{kozachinskiy2026parity}.

We train those tasks under the same protocol. The empirical results are in line with our NESV analysis: the task-model pair whose NESV shrinks faster as $n$ grows shows worse length generalization. The only counterexample was that, under the general sinusoidal model, MAJORITY trains as well as under the position-free model $\Acount$, although the former's volume shrinks faster with $n$.

\begin{table}
\caption{NESV bounds and mean accuracy at $10\times$ the maximum training length for the four tasks. 
The models $\Acount$, $\Asin$, and $\Apos$ are described in \cref{sec:models}, while $\Anorm$ is described in \cref{sec:index-norm}.}
\medskip
\label{tab:summary}
\centering
\small
\setlength{\tabcolsep}{0pt}
\begin{tabular}{@{}lc@{~}r@{\hspace{8pt}}cr@{\hspace{8pt}}cr@{\hspace{8pt}}cr@{}}
\toprule
Task & \multicolumn{8}{c}{Model} \\
& \multicolumn{2}{c}{Positional ($\Apos$)}
& \multicolumn{2}{c}{General ($\Acount$)}
& \multicolumn{2}{c}{Sinusoidal ($\Asin$)}
& \multicolumn{2}{c}{Normalized ($\Anorm$)}
\\
\cmidrule(lr){2-3}
\cmidrule(lr){4-5}
\cmidrule(lr){6-7}
\cmidrule(l){8-9}
& NESV & Acc. (\%)
& NESV & Acc. (\%)
& NESV & Acc. (\%)
& NESV & Acc. (\%)
\\
\midrule
\FIRST
&
$\Theta(1)$&
100.0&
$0$&
51.12&
$0$&
55.06&
$\Theta(1)$&
100.0\\

\MAJORITY
&
---&
92.7&
$\Theta((n\log n)^{-1})$&
82.6&
 $\Theta(({n^3\log n})^{-1})$&
87.5&
---&
92.7\\

\INDEX
&
$\Theta(n^{-3})$&
59.6&
$0$&
51.9&
$0$&
50.2&
$\Theta(n^{-1})$&
85.1\\

\PARITY
&
$0$&
---&
$0$&
---&
$0$&
---&
$0$&
---\\
\bottomrule
\end{tabular}
\end{table}

Among the tasks we examine, \INDEX{} demonstrates particularly bad length generalization (60\% average accuracy at ten times training length) despite being expressible \citep{kozachinskiy2026indexing}. Our NESV analysis reveals the underlying mechanisms. Exact retrieval of a designated position $k$ requires that position to receive enough attention, so its attention score must remain within a constant of the maximum attention score.  Expressing the difference in attention scores between neighboring positions as a function of~$k$ reveals two key error components, one proportional to~$k^2$ and another proportional to~$k$, which are the sources of failing length generalization. 

Guided by this analysis, we modify the model so that the error term proportional to $k^2$ no longer grows with input length. This is called the normalized positional model. The NESV improves from $\Theta(n^{-3})$ to $\Theta(n^{-1})$. In our experiments, the modified model also generalizes substantially better on \INDEX{}, achieving an average accuracy of 85\% percent at ten times the training length.

Taken together, the results show that expressivity alone can conceal qualitatively different parameter-space behavior. Volume analysis, on the other hand, may offer another useful tool to further identify the computational constraints responsible for learning behaviors and offer insight for task-specific model refinements.

\section{Setup}
\label{sec: setup}
In this section, we describe the tasks to be analyzed and the transformers to be used in the analysis.

\subsection{Tasks and exact-solution volume}

For $n\ge1$, write $[n]=\{1,\ldots,n\}$ and let $w=(w_1,\ldots,w_n)\in\{0,1\}^n$. We study
\begin{align*}
\FIRST_n(w)&=w_1 \\
\INDEX_n(w,k)&=w_k,\quad k\in[n],\\
\MAJORITY_n(w)&=\1\mleft\{\sum_{j=1}^n w_j\ge\frac{n+1}{2}\mright\}
\quad\text{for odd }n,\\
\PARITY_n(w)&=\left(\sum_{j=1}^n w_j\right)\bmod 2.
\end{align*}
For \INDEX, the $n$ tokens $w_1, \ldots, w_n$, which we call \emph{data} tokens, are followed by a single token, which we call the \emph{index} token, encoding $k$. We refer to this as \emph{right-hand sequence \INDEX{}ing}. For the other tasks, there is an index token for consistency's sake, but it is fixed.

\begin{definition}[Normalized exact-solution volume, NESV]\label{def:volume}
For a task $T$ and transformer model $\mathcal A$, let $S_n^{T,\mathcal A}$ be the parameters for which the output is strictly correct on every length-$n$ input, and on every index $k$. For a fixed parameter bound $B>0$, define
\begin{equation}
\Vol^{T,\mathcal A}(n)
=
\frac{\lambda_{D_{\mathcal A}}\mleft(S_n^{T,\mathcal A}\cap[-B,B]^{D_{\mathcal A}}\mright)}{(2B)^{D_{\mathcal A}}}.
\label{eq:volume}
\end{equation}
\end{definition}
Equivalently, \eqref{eq:volume} is the probability of obtaining a correct solution when each scalar parameter is sampled independently and uniformly from \([-B,B]\). Unless stated otherwise, $B$ and the hidden width of the FFNN, $h$, are fixed as $n$ grows. Throughout, $\log$ denotes the natural logarithm; $\|\mathord\cdot\|_2$ and $\|\mathord\cdot\|_\infty$ denote the Euclidean and maximum-coordinate norms. Throughout, $c,C>0$ and the constants implicit in the subscripted asymptotic notation such as $\Theta_{B,h}$ may depend on $B$ and $h$, but never on $n$; $c$, $C$ may change from line to line. 

\subsection{Transformer Models}
\label{sec:models}

We consider only single-layer transformers.
Let \(x_j\) denote the \(j\)-th data token and
let \(z\) denote the index token. The query, key, and value maps have the form $Q(x)=W_Qx+b_Q,
K(x)=W_Kx+b_K,
V(x)=W_Vx+b_V.$ Writing $q=Q(z),  k_j=K(x_j), v_j=V(x_j)$,  the attention output is
\begin{equation}
\alpha_j
=
\frac{\exp\!\left((\log n)q^\top k_j\right)}
{\sum_{i=1}^n\exp\!\left((\log n)q^\top k_i\right)},
\qquad
A
=
\sum_{j=1}^n\alpha_j u_j.
\label{eq:attention}
\end{equation}
The $\log n$ factor was proposed by \citet{chiang2022overcoming} and \citet{nakanishi2025scalablesoftmaxsuperiorattention} and has been used in practice in Qwen \citep{bai2023qwentechnicalreport}. \Citet{chen-etal-2026-critical} provide additional theoretical justification for this choice.  In this paper, this choice is natural since we study positional tasks such as \INDEX{}, where attention must continue to select a particular position as the number of competing positions grows with the input length. We omit the \(1/\sqrt d\) factor in conventional $\log n$ scaling because \(d\) is constant in \(n\) and does not affect the asymptotic bounds in $n$.

The self-attention is followed by a feed-forward network (FFNN) classifier with hidden width $h$:
\begin{equation}
y
=
W_2\relu(W_1A+b_1)+b_2
\in\R^2,
\label{eq:output}
\end{equation}
where $W_1\in\R^{h\times2},
b_1\in\R^h,
W_2\in\R^{2\times h},
b_2\in\R^2$.

The two coordinates of $y$ are indexed by the labels \(\{0,1\}\), and the predicted label is $\widehat y
=
\operatorname*{argmax}_{c\in\{0,1\}}(y)_c$. 
When the input needs to be explicit, we write \(A(w,k)\) and \(y(w,k)\). 

We do not use residual connections or layer normalization. %
Adding them could alter the size and geometry of the exact-solution set; %
we leave this extension for future work.

In this paper, we consider four models. We describe three of them here; the fourth, $\Anorm$, we leave for \cref{sec:index-norm}.

In the \textbf{general} model $\Acount(h)$, the data tokens are $x_j=e_{w_j}\in\R^2$, and the index token is $z=0_2$. The maps $Q,K,V:\R^2\to\R^2$ are unrestricted affine
maps, so $q=Q(0_2)$, $k_j=K(e_{w_j})$, and
$v_j=V(e_{w_j})$. So $\Acount(h)$ has $D_{\mathrm{count}}=5h+20$ trainable scalar parameters. 

We also examined a model, $\Asin(h)$, with a \textbf{sinusoidal} positional encoding ($\sin(j)$, $\cos(j)$).

For the positional tasks, we use the \textbf{positional model} \(\Apos(h)\). The data token at position \(j\) and the index token
encoding \(k\) are
\[
x_j^{\mathrm{pos}}
=
\begin{bmatrix}
e_{w_j}\\
j\\
j^2
\end{bmatrix}
\in\R^4,
\qquad
z_k^{\mathrm{pos}}
=
\begin{bmatrix}
0_2\\
k\\
0
\end{bmatrix}
\in\R^4
\]
(where $0_2$ is the two-dimensional zero vector).
For \FIRST, the index is fixed at \(k=1\). For \INDEX, the index \(k\in[n]\) is supplied as part of the input. 
Because \(j\) and \(j^2\) grow without bound in the input length, allowing them to propagate through the network would like lead to instability;
positional-encoding choices are known to strongly affect length extrapolation \citep{kazemnejad2023impact}. 
Accordingly, we constrain the model to have the form of the index-lookup construction of \citet{bhattamishra2024separations}: %
\begin{align*}
W_Q^{\mathrm{pos}}
&=
\begin{bmatrix}
0_{2\times2} & a & 0_2
\end{bmatrix},
&
b_Q^{\mathrm{pos}}&=b, \\
W_K^{\mathrm{pos}}
&=
\begin{bmatrix}
0_{2\times2} & r & s
\end{bmatrix},
&
b_K^{\mathrm{pos}}&=c_K,
\\
W_V^{\mathrm{pos}}
&=
\begin{bmatrix}
C_V & 0_{2\times2}
\end{bmatrix},
&
b_V^{\mathrm{pos}}&=b_V,
\end{align*}
where $a,b,r,s,c_K,b_V\in\R^2,
C_V\in\R^{2\times2}$ (and $0_{2\times2}$ is the $2 \times 2$ zero matrix).
This model has $D_{\mathrm{pos}}=5h+18$ trainable scalar parameters.

\section{NESV bounds and length generalization}
\label{sec:comparison}

We first establish the volume bounds for the four tasks and then compare such results with those of length-generalization experiments. 

\subsection{Main result}

\begin{theorem}[NESV bounds]\label{thm:nesv-bounds}
Fix $B>0$ and $h\ge2$. As $n\to\infty$,
\begin{align}
\Vol^{\FIRST,\Apos(h)}(n)&=\Theta_{B,h}(1)\label{eq:first-main}\\
\Vol^{\INDEX,\Apos(h)}(n)&=\Theta_{B,h}(n^{-3})\label{eq:index-sharp}\\
\Vol^{\MAJORITY,\Acount(h)}(n)&=\Theta_{B,h}\bigl((n\log n)^{-1}\bigr)
&&\text{$n$ odd} \label{eq:maj-sharp} \\
\Vol^{\PARITY,\Acount(h)}(n)&=0 && n>h+1.
\label{eq:parity-zero}
\end{align}
\end{theorem}

\begin{proof}[Proof sketch]
For \FIRST, one can choose parameters such that the attention score decreases quadratically away from position one. A fixed neighborhood of these parameters preserves that decrease, and the resulting concentration of attention lets the classifier return the first symbol for all sufficiently large lengths. The neighborhood does not shrink with $n$, thus ensuring a constant NESV bound. 

However, the same argument cannot be used for \INDEX, because the attention score must peak at $j=k$ for every possible $k$, and small errors that are harmless at one position can be amplified at another. We examine the two relevant error sources in \cref{sec:index}.

In the position-free model $\Acount$, a word is represented only by its number of 1's. For \MAJORITY, the output changes only once,
between counts \((n-1)/2\) and \((n+1)/2\). These counts differ by just one, so the interval that separates them narrows as $n$ increases. Consequently, placing the decision boundary between them requires increasingly precise parameter choice. 

For \PARITY, by contrast, every increment in the count must flip the output. A fixed number of hidden ReLU units cannot provide arbitrarily many flips in output. The upper bound of zero, meaning that \PARITY{} is not solvable by the class of transformers considered here, is consistent with previous theoretical results \citep{kozachinskiy2026parity} and experiments \citep{chiang2022overcoming,hahn2024sensitive}.

The complete proofs of the four bounds, including the technical details omitted from the \INDEX\ proof sketch in \cref{sec:index-components}, are given in \cref{app:baseline-proofs}. 
\end{proof}
For detailed results related to the  sinusoidal model $\Asin$, see \cref{app:sinusoidal-majority}.

\subsection{Experiment}

To compare these bounds with empirical results, we train the \FIRST, \INDEX, and \MAJORITY\ models with hidden width $h=8$. We use 20 random seeds per task and draw a fresh batch at each update, sampling its length uniformly from $1,\ldots,N$; for \MAJORITY, we use only odd lengths. The main results use $N=51$. Checkpoints are selected using validation data at length $N$, without inspecting longer inputs, and tested on random examples at lengths up to 511. Additional experiments on $N=11,21,31,41$, which give similar results, are in \cref{app:experiments}.

\Cref{tab:summary} shows that \FIRST\ remains accurate at the longest test length (under $\Apos$), \MAJORITY\ retains substantial accuracy (under $\Acount$), and \INDEX\ falls much closer to chance (under $\Apos$). These observations agree with the ranking of the three volume bounds. 

We found the results of \INDEX{} particularly interesting. Despite the task being expressible and learnable on lengths up to 51, it showed poor length generalization, much worse than that of \MAJORITY{}.  We take a closer look at this task in the next section.

\section{The \INDEX\ task}
\label{sec:index}

We first revisit the volume argument from the previous section, to identify the parameter combinations responsible for the $n$-dependent decay. We then modify the model to
control one of these combinations structurally and compare the
resulting theoretical and empirical behavior.

\subsection{Length-sensitive components of the attention score}
\label{sec:index-components}

Recall that, in the positional model $\Apos(h)$,  $q=ak+b, 
k_j=rj+sj^2+c_K, 
u_j=C_Ve_{w_j}+b_V.
$
The attention score is
\begin{equation}
g_j=q^\top k_j = (ak+b)^\top(rj+sj^2+c_K).
\label{eq:index-score}
\end{equation}
We compare the scores at position $k$ and position $k+\Delta$. The expansion will produce the following quantities used repeatedly in later proofs, so we introduce some notations for future convenience.
\begin{equation}
\chi=a^\top s
\qquad
\ell=2b^\top s+a^\top r
\qquad
\delta=b^\top s
\qquad
\phi=b^\top r.
\label{eq:index-effective-coefficients}
\end{equation}
Now consider position $k+\Delta \in [n]$. The \emph{score gap} of this position is
\begin{equation}
\begin{aligned}
g_{k+\Delta}-g_k
&=
\Delta
\underbrace{
  \left(2\chi k^2+\ell k+\phi\right)
}_{C_{\mathrm{lin}}(k)}
+
\Delta^2
\underbrace{
  \left(\chi k+\delta\right)
}_{C_{\mathrm{quad}}(k)}.
\end{aligned}
\label{eq:index-score-displacement}
\end{equation}
We want this gap to be negative. The coefficient $C_{\mathrm{quad}}(k)$ controls the curvature of the score around the indexed position, while $C_{\mathrm{lin}}(k)$ measures its first-order drift and therefore whether the score remains centered at $j=k$; we want it to be zero. Because $\chi$ and $\ell$ are multiplied by $k$ and $k^2$, small errors in these parameter combinations can be amplified at large $k$ (i.e., at large $n$).

\subsubsection{Sketch of lower bound}

See \cref{sec:index_lower} for the full proof of the lower bound.

Let $\eta=\min\{B/16,1/16\}$ and consider the solution (which we call the \emph{center}):
\begin{equation}
a^0=(2\eta,0),
\qquad
b^0=(0,-\eta),
\qquad
r^0=(\eta,0),
\qquad
s^0=(0,\eta).
\label{eq:index-routing-center}
\end{equation}
We want to find a region around the center that still solves the problem. The volume of this region will serve as our lower bound. At this center, $\chi=\ell=\phi=0$, $\delta=-\eta^2$, and the score gap is 
\begin{equation}
g_{k+\Delta}-g_k=-\eta^2\Delta^2.
\label{eq:index-center-gap}
\end{equation}
In a sufficiently small neighborhood of the center, we may restrict
\begin{equation}
|\chi|\le c n^{-2},
\qquad
|\ell|\le c n^{-1},
\label{eq:index-lower-widths}
\end{equation}
while allowing $a$, $b$, $r_2$, and $s_2$ (the second coordinates of $r$ and $s$, resp.) to vary over fixed-width intervals independent of $n$. Every parameter choice in this region then satisfies $g_{k+\Delta}-g_k
\le
-\frac{\eta^2}{4}\Delta^2$ for every $k\in[n]$ and every nonzero $\Delta$ such that $k+\Delta\in[n]$. Under the $\log n$ attention scale, the total attention outside the indexed position then tends to zero as $n \to \infty$ uniformly in $k$.  Hence, for all sufficiently large $n$, the indexed position receives at least a fixed fraction of the total attention, uniformly over all inputs and indexes. We may then restrict $(C_V,b_V)$ and $(W_1,b_1,W_2,b_2)$ to fixed-size neighborhoods, independent of $n$, such that, for some constant $c_0>0$ independent of $n$, $y_{w_k}(w,k)-y_{1-w_k}(w,k)\ge c_0$ for every input $(w,k)$. To produce a exact solution, only the width of the interval $\chi$ and $\ell$ need to depend on $n$. Therefore the constructed set of exact solutions has normalized volume
\begin{equation}
\Omega(n^{-2})\cdot\Omega(n^{-1})
=
\Omega(n^{-3}).
\label{eq:index-lower-volume-sketch}
\end{equation}

\subsubsection{Sketch of upper bound}

See \cref{sec:index_upper} for the full proof of the upper bound.

For each indexed position $k$, compare a constant word with a word that differs from it only at position $k$. These two inputs must be assigned different labels. 
We consider two cases. Let $m = y_{w_k}(w,k)-y_{1-w_k}(w,k)$. In any correct solution, we will have $m > 0$,  but if $m \le n^{-5}$, we call the solution \emph{small-margin}; otherwise, we call it \emph{large-margin}.

Take any one of the inequalities $y_{w_k}(w,k)>y_{1-w_k}(w,k)$; there are $4n$ such inequalities in total.  Fix arbitrary values of $a,b,r,s,c_K,C_V,b_V,W_1,b_1,W_2$, and $b_{2,1-w_k}$. As $b_{2,w_k}$ varies, the margin of the inequality varies with slope one. To satisfy $0<y_{w_k}(w,k)-y_{1-w_k}(w,k)\le n^{-5}$, $b_{2,w_k}$ needs to lie in an interval of length $n^{-5}$. Thus the set of small-margin solutions has normalized volume at most $4n\cdot n^{-5}/(2B)=O(n^{-4})$.

For every large-margin solution, changing only the symbol at position $k$ from $\tau$ to $\sigma\ne\tau$ must change the predicted label from $\tau$ to $\sigma$. By the definition of a large-margin solution, for each of the two words
$w$ in this comparison, $y_{w_k}(w,k)-y_{1-w_k}(w,k)>n^{-5}$. Since the $y_{\sigma}(w,k)-y_{\tau}(w,k)$,
viewed as a function of the attention output $A(w,k)$ is Lipschitz, this is impossible if the attention assigned to position $k$ is too small. Hence the indexed position receives at least an inverse polynomial amount of attention. The softmax formula in \eqref{eq:attention} then implies that $g_k$ lies within a constant of the maximum score. Comparing it with the two neighboring scores in \eqref{eq:index-score-displacement} gives $C_{\mathrm{quad}}(k)+|C_{\mathrm{lin}}(k)|\le M_0 ,$ where $M_0$ is independent of $n$. Since every coordinate of $a,b,s$ lies in $[-B,B]$, both
$\chi=a^\top s$ and $\delta=b^\top s$ are bounded independently of $n$. Hence, for $k\le n$, $|C_{\mathrm{quad}}(k)| = |\chi k+\delta| = O(n)$, and thus $|C_{\mathrm{lin}}(k)| \le M_0+|C_{\mathrm{quad}}(k)| = O(n).$ Let $M=\lfloor n/4\rfloor$. Since $C_{\mathrm{lin}}(k)=2\chi k^2+\ell k+\phi$, we have $C_{\mathrm{lin}}(3M)-2C_{\mathrm{lin}}(2M)+C_{\mathrm{lin}}(M)=4\chi M^2$. Because each term on the left is $O(n)$ and $M=\Theta(n)$, this gives $|\chi|=O(n^{-1})$. Moreover, $C_{\mathrm{lin}}(2M)-C_{\mathrm{lin}}(M)=6\chi M^2+\ell M$. Using $|\chi|=O(n^{-1})$ in this identity gives $|\ell|=O(1)$.
Repeating the same argument gives
\begin{equation}
|\chi|=O(n^{-2}),
\qquad
|\ell|=O(n^{-1}).
\label{eq:index-sharp-coefficient-constraints}
\end{equation}
We used the fact 
\begin{equation}
    \Pr\left(
 |a^\top s|\le\varepsilon_1,
 \ |2b^\top s+a^\top r|\le\varepsilon_2
\right)=O_B(\varepsilon_1\varepsilon_2)
\end{equation}
whenever $0<\varepsilon_1\le\varepsilon_2$ are sufficiently small (for detailed proof see \cref{sec:index_upper}, \cref{lem:index-joint-slab}). Apply this fact to \eqref{eq:index-sharp-coefficient-constraints} and adding the $O(n^{-4})$ small-margin volume would gives the upper bound $O(n^{-3})$. 

\subsubsection{Experimental analysis}

We next check whether $\chi=a^\top s  \text{ and }  \ell=2b^\top s+a^\top r$  identified by the volume proof track the behavior of trained models. In \eqref{eq:index-score-displacement}, $\chi$ controls the error term whose effect grows quadratically with $k$, while $\ell$ controls the error term whose effect grows linearly with $k$. For a model trained up to maximum length $N$, we therefore consider the quantities $N^2|\chi|$ and $N|\ell|$, allowing models trained at different maximum lengths to be comparable. Across the 100 positional models trained at five maximum lengths, smaller values of both quantities are associated with higher
mean accuracy on unseen lengths; see \cref{fig:index-coefficients}. Ranking the models separately within each maximum training length gives correlations of $-0.655$ for $N^2|\chi|$ and $-0.516$ for $N|\ell|$. Both have $p<10^{-3}$ under permutation tests that preserve the five training-length groups. The training and evaluation protocol is given in \cref{app:training-protocol,app:evaluation-banks}, and the coefficient analysis is detailed in \cref{app:index-coefficient-experiments}.

\begin{figure}
\centering
\vspace*{-2ex}
\includegraphics[scale=0.5]{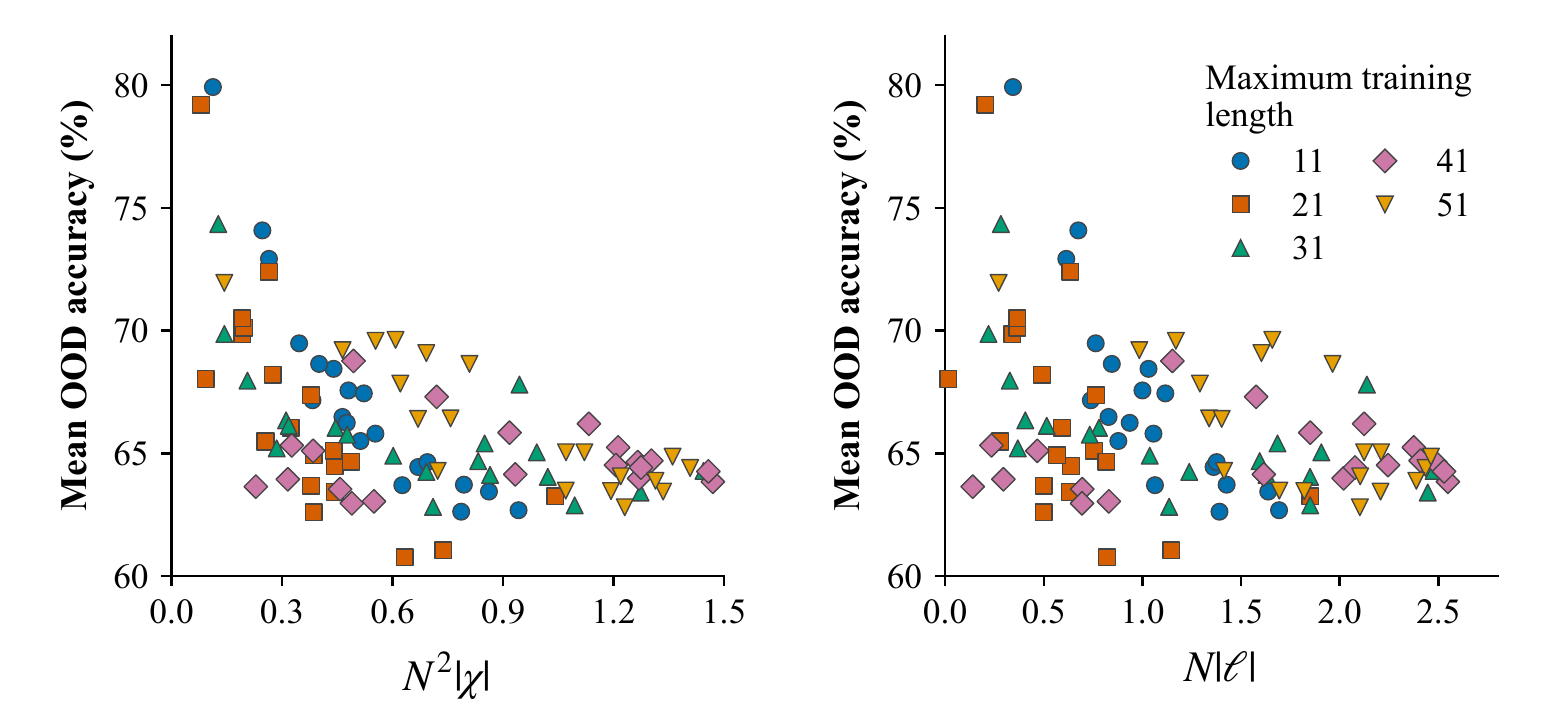}
\vspace*{-2ex}
\caption{Scale-normalized drift coefficients and mean OOD accuracy for trained positional \INDEX\ models. Markers indicate the maximum training length $N$.}
\label{fig:index-coefficients}
\end{figure}

The above analysis suggests a direct architectural change: instead of relying on training to place two free coefficients inside shrinking intervals, control one of the length-growing error terms through the model itself.

\subsection{The Normalized Model}
\label{sec:index-norm}

We introduce a modified positional model, denoted $\Anorm(h)$. It uses the same input vectors, value map, ReLU classifier, and trainable positional parameters $a,b,r,s,c_K$ as $\Apos(h)$. 
We split the computation of the attention score into two \emph{channels}:  a linear channel that does not depend on the $j^2$ coordinate of the keys' positional encoding, and a quadratic channel that depends only on the $j^2$ coordinate.
In the quadratic channel, we modify the queries as follows. Define the fixed smooth map
\begin{equation}
F(v)=\frac{v}{\sqrt{1+\|v\|_2^2}}
\qquad (v\in\mathbb R^2).
\label{eq:norm-map}
\end{equation}
The linear and quadratic positional channels are
\begin{equation}
q^{\mathrm{lin}}=ak+b,
\qquad
q^{\mathrm{quad}}=b+F(ak),
\qquad
k_j^{\mathrm{lin}}=rj+c_K,
\qquad
k_j^{\mathrm{quad}}=sj^2.
\label{eq:norm-channels}
\end{equation}
Thus the linear channel uses the full query $ak+b$, whereas the quadratic channel applies $F$ to the index-dependent contribution $ak$. The attention score is
\begin{equation}
\begin{aligned}
g_j^{\mathrm{norm}}
&=
\bigl(q^{\mathrm{lin}}\bigr)^\top k_j^{\mathrm{lin}}
+
\bigl(q^{\mathrm{quad}}\bigr)^\top k_j^{\mathrm{quad}}
&=
(ak+b)^\top(rj+c_K)
+
\bigl(b+F(ak)\bigr)^\top sj^2.
\end{aligned}
\label{eq:norm-score}
\end{equation}
The attention weights are obtained by applying the same $\log n$-scaled softmax to \eqref{eq:norm-score}; the values and the one-hidden-layer ReLU classifier are unchanged. The normalization introduces no new trainable parameters, so $\Anorm(h)$ has the same parameter dimension and uses the same $B$ as $\Apos(h)$.

This modification could be applied to a general transformer by excluding $j^2$ from the position encoding at position $j$ but adding an additional term $f(q_i,k_j)j^2$ directly into the attention score, with $f$ being a scalar-valued function whose output remains bounded. We leave exploration of this more general setting for future work.

The map $F$ is approximately the identity near the origin, 
but its output has norm strictly smaller than one. Consequently, the component of the query that depends on $k$ and is multiplied with $j^2$ no longer grows linearly with $k$. To make this precise, suppose first that $a\ne0$ and write $\widehat a=a/\|a\|_2$.  Retaining $\delta$ and $\phi$ from
\eqref{eq:index-effective-coefficients}, define
\begin{equation}
\widehat\chi:=\widehat a^\top s,
\qquad
\widehat\ell:=\ell+2\widehat\chi,
\qquad
\varepsilon_\chi(k):=
\bigl(F(ak)-\widehat a\bigr)^\top s.
\label{eq:norm-effective-terms}
\end{equation}
Then the modified quadratic and linear coefficients are
\begin{equation}
\begin{aligned}
C_{\mathrm{quad}}^{\mathrm{norm}}(k)
&=
\widehat\chi+\delta+\varepsilon_\chi(k),
\\
C_{\mathrm{lin}}^{\mathrm{norm}}(k)
&=
\widehat\ell k+\phi+2k\varepsilon_\chi(k).
\end{aligned}
\label{eq:norm-gap-coefficients}
\end{equation}
Thus,
\begin{equation}
g_{k+\Delta}^{\mathrm{norm}}-g_k^{\mathrm{norm}}
= \Delta C_{\mathrm{lin}}^{\mathrm{norm}}(k)+ \Delta^2 C_{\mathrm{quad}}^{\mathrm{norm}}(k).
\label{eq:norm-score-displacement}
\end{equation}
Moreover,
\begin{equation}
|\varepsilon_\chi(k)|
\le
\frac{\|s\|_2}{2k^2\|a\|_2^2},
\qquad
|2k\varepsilon_\chi(k)|
\le
\frac{\|s\|_2}{k\|a\|_2^2}.
\label{eq:norm-remainder-bound}
\end{equation}
Thus the $k^2$ drift in \eqref{eq:index-score-displacement} is replaced by a linear term and a bounded error $2k\varepsilon_\chi(k)=O(k^{-1})$. Only the coefficient $\widehat\ell$ remains to be confined to an interval of width $O(n^{-1})$.

\begin{theorem}[NESV of \INDEX{} under $\Anorm(h)$]
\label{thm:index-norm}
For $B>0$ and $h\ge2$. \begin{equation}
\Vol^{\INDEX,\Anorm(h)}(n)
=
\Theta_{B,h}(n^{-1}).
\label{eq:index-norm-volume}
\end{equation}
\end{theorem}

\begin{proof}
The proof follows the same strategy as that of the \INDEX{} bounds for the positional model in \cref{sec:index-components}. The full proof is given in \cref{app:norm}.
\end{proof}

Compared with the original positional model, selective normalization removes one shrinking direction from the volume calculation. In $\Apos(h)$, the two effective constraints have widths $n^{-2}$ and $n^{-1}$; in $\Anorm(h)$, the saturation of the quadratic query channel leaves only the $n^{-1}$ constraint. %
The improvement from $\Theta(n^{-3})$ to $\Theta(n^{-1})$ therefore follows from the same score-level mechanism identified by the original proof.

\subsection{Length-generalization experiments}
\label{sec:index-norm-experiments}

We train $\Apos(h)$ and $\Anorm(h)$ under the same protocol, using paired seeds, the same training batches, and identical initial values for all parameters shared by the two models. The main comparison uses classifier width $h=8$, maximum training length $N=51$, and 20 seeds.

As shown in \cref{fig:index-norm-generalization}, the original model degrades rapidly beyond the training range, whereas the selectively normalized model retains substantially higher accuracy. At length 511, approximately ten times the maximum training length, mean accuracy increases from $60\%$ for $\Apos(h)$ to
approximately $85\%$ for $\Anorm(h)$. Experiments with maximum training lengths $N\in\{11,21,31,41\}$ show similar results and are reported in \cref{app:index-norm-experiments}.

\begin{figure}
\centering
\vspace*{-2ex}
\includegraphics[scale=0.65]{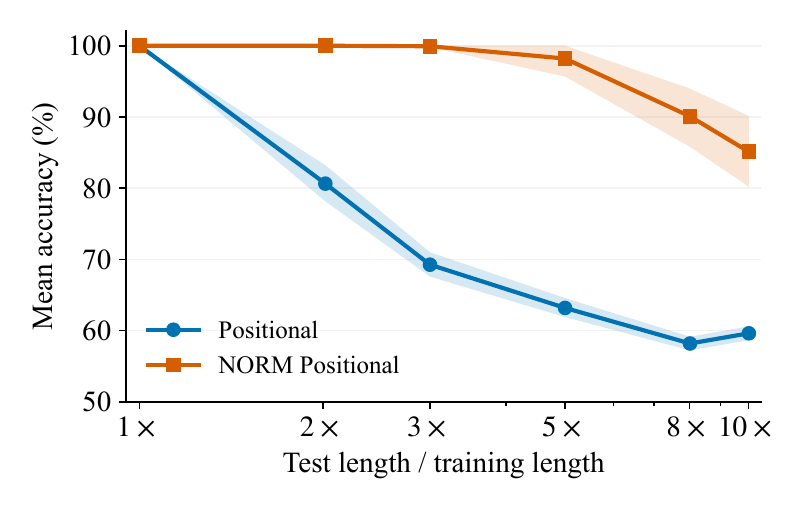}
\vspace*{-2ex}
\caption{Length generalization on \INDEX\ for the positional and $\Anorm(h)$ models trained up to $N=51$. Shading shows $95\%$ bootstrap confidence intervals. }
\label{fig:index-norm-generalization}
\end{figure}

Taken together, the coefficient analysis and the architectural intervention give two complementary forms of evidence. The original volume proof identifies the parameter combinations whose effects grow with the index $k$, and those combinations are associated with the behavior of trained models. Selectively normalizing the quadratic query channel removes the faster-growing contribution, enlarges the exact-solution volume from $\Theta(n^{-3})$ to $\Theta(n^{-1})$, and improves extrapolation in the model for which the volume calculation is carried out.

\section{Related Work}

The \INDEX{} task has been studied many times before.
\citet{bhattamishra2024separations} give an efficient one-layer construction in which the data appear before the indexed \INDEX{}, as in our right-hand \INDEX\ task. This differs from left-hand \INDEX{}ing, where the \INDEX{} appears before the data and the answer is produced later \citep{kozachinskiy2025strassen,kozachinskiy2026indexing}. One-layer constructions for related right-hand function-evaluation problems are also studied by \citet{strobl2025concise}.

A line of work related to our normalized model $\Anorm$ controls attention by normalizing query--key interactions. QKNorm $\ell_2$-normalizes complete query and key vectors and learns the attention scale \citep{henry-etal-2020-query}, while scaled cosine attention likewise removes dependence on their magnitudes \citep{liu2022swin}. Separately, approaches to length extrapolation often modify the positional contribution to attention, for example by adding distance-dependent score biases \citep{press2022train} or by designing relative position encoding to improve attention resolution beyond the training lengths \citep{sun-etal-2023-length}. Our $\Anorm(h)$, on the other hand, does not normalize the full query or key vectors, nor the quadratic positional coordinate $j^2$. Instead, it applies a fixed smooth bounded map to the index-dependent query component in the quadratic positional channel. %

Another related line of work studies which functions occupy large or stable regions of transformer parameter space. \citet{hahn2024sensitive} connect Boolean sensitivity to transformer norms and perturbation behavior, while \citet{kover2026geometry} study the parameter-space geometry of transformers computing Boolean functions. Empirical work on simplicity bias likewise suggests that the parameter-to-function map favors some Boolean functions over others \citep{bhattamishra2023simplicity}.
Our volume measure tracks how volume changes with sequence length, rather than assigning a single complexity or sensitivity measure to the target task. For \INDEX, such analysis not only shows that the volume decreases with longer input lengths, but also isolates the sources of such shrinkage.

\section{Discussion and Limitations}

\paragraph{Initialization and optimization}

Parameter-space volume depends on the distribution placed on the parameters.
\Cref{eq:volume} assumes a uniform distribution over a box of size $2B$ in the raw coordinates. 
But initialization and optimization can concentrate strongly on particular parts of parameter space, so differences in exact-solution volume might not translate exactly into differences in  accuracy.
Tightening the correspondence would therefore require combining the present analysis with the distributions over parameters induced by initialization and optimization.

\paragraph{Depth}

Our theoretical results concern fixed-width, one-layer, one-head transformers. This restriction can matter qualitatively. \citet{bhattamishra2024separations} show examples where two-layer transformers of logarithmic size solve problems for which one-layer transformers require linear size. The parameter constraints governing length generalization may therefore change substantially with depth, and extending the volume analysis to deeper transformers is an important next step.

\section{Conclusion}

Parameter-space volume is useful not only for measuring how exact solutions change with input length, but also for revealing \emph{why}\ they change. On the \INDEX{} task, our volume analysis isolates the length-sensitive components, suggesting a modification to the model itself that improves length extrapolation. This suggests a broader use of volume analysis as a diagnostic for transformer model refinement, especially task-specific ones.

\section*{Reproducibility Statement}
The baseline architectures, parameter prior, and strict-correctness criterion are defined in \cref{sec: setup}. The $\Anorm(h)$ architecture is defined in \cref{sec:index-norm}, and the fixed-sinusoidal \MAJORITY\ model is defined in \cref{app:sinusoidal-majority-model}. The complete volume arguments, experimental protocol, checkpoint-selection rule, evaluation-bank construction, and supplementary results are presented in the appendix.

\section*{AI Use Statement}
In this work, we used generative AI tools to aid literature exploration, assist in the writing of proofs (prose editing, LaTeX preparation, proofreading), create or modify scientific figures, and edit code (for experiments). We have not used generative AI tools to develop theoretical models or conceptual frameworks, formulate mathematical claims, provide critical ingredients for proving mathematical claims, propose or refine hypotheses, design or provide feedback on research methodology or experiments, or assist with translation. %
We have reviewed all AI-assisted work and take responsibility for the final content of this work, including text, claims or artifacts produced with the aid of generative~AI.

\bibliographystyle{iclr2027_conference}
\bibliography{references}

\crefalias{section}{appendix}

\appendix

\section{Proofs of the Baseline NESV Bounds}
\label{app:baseline-proofs}

This appendix proves the four baseline bounds in \cref{thm:nesv-bounds}.  

\subsection{\FIRST}
\label{app:first}

The lower bound constructs fixed neighborhoods of parameters independent of input length that gives exact solutions to the task.  The upper bound follows from symmetry between the two labels.

\begin{proof}[Proof of the \FIRST\ bound in \cref{thm:nesv-bounds}]
Choose $\eta=\min\{B/16,1/16\}$ and set
\begin{equation}
  a^0=(2\eta,0),
  \qquad
  b^0=(0,-\eta),
  \qquad
  r^0=(\eta,0),
  \qquad
  s^0=(0,\eta),
  \label{eq:first-center}
\end{equation}
with $c_K^0=b_V^0=0$ and $C_V^0=\eta I_2$. We call this set of $(a^0, b^0, r^0, s^0) $ the center.  Put $  q_F:=a+b$,  $\widehat\ell_1:=q_F^\top r$,  $\widehat\ell_2:=q_F^\top s.$
At the center, $\widehat\ell_1^0=2\eta^2$ and $\widehat\ell_2^0=-\eta^2$.  By continuity, there is a fixed positive-volume neighborhood of $(a,b,r,s)$ on which
\begin{equation}
  |\widehat\ell_1-2\eta^2|\le \frac{\eta^2}{8},
  \qquad
  |\widehat\ell_2+\eta^2|\le \frac{\eta^2}{8}.
  \label{eq:first-zeta-neighborhood}
\end{equation}
For $j=1+\Delta$ with $\Delta\ge1$, the key bias cancels and
\begin{align}
  g(1)-g(1+\Delta)
  &=\Delta\bigl[(-\widehat\ell_1-2\widehat\ell_2)+(-\widehat\ell_2)\Delta\bigr]
  \notag\\
  &\ge \frac{\eta^2}{2}\Delta^2.
  \label{eq:first-score-gap}
\end{align}
Consequently, throughout this fixed positional-parameter neighborhood,
\begin{equation}
  1-\alpha_1
  \le
  \sum_{\Delta=1}^{\infty}
  n^{-(\eta^2/2)\Delta^2}
  \longrightarrow 0.
  \label{eq:first-attention-concentration}
\end{equation}
Restrict $(C_V,b_V)$ to a fixed neighborhood on which
\begin{equation}
  \|C_Ve_\sigma+b_V-\eta e_\sigma\|_\infty
  \le \frac{\eta}{16}
  \qquad (\sigma\in\{0,1\}).
  \label{eq:first-value-neighborhood}
\end{equation}
For all sufficiently large $n$, \eqref{eq:first-attention-concentration} gives $\alpha_1\ge7/8$.  If $y^\star=w_1$ and $\bar y^\star=1-y^\star$, the total attention masses on the two symbols satisfy
\begin{equation}
  p_{y^\star}-p_{\bar y^\star}\ge \frac34.
\end{equation}
Convexity of the value perturbation in \eqref{eq:first-value-neighborhood} therefore gives
\begin{equation}
  A_{y^\star}-A_{\bar y^\star}
  \ge
  \eta\bigl(p_{y^\star}-p_{\bar y^\star}\bigr)-\frac{\eta}{8}
  \ge \frac{5\eta}{8}.
  \label{eq:first-value-margin}
\end{equation}

Choose \(W_1^0,b_1^0,W_2^0,b_2^0\) as follows. Take the first two rows of $W_1^0$ to be $\eta e_0^\top$ and $\eta e_1^\top$, with both corresponding biases equal to $\eta^2$.  Set every remaining row of $W_1^0$ to zero and its bias to $-\eta$.  In $W_2^0$, give output row zero coefficient $\eta$ on hidden unit zero and output row one coefficient $\eta$ on hidden unit one, with all other entries zero; set $b_2^0=0$.  The first two hidden units are active throughout a sufficiently small neighborhood of \eqref{eq:first-value-neighborhood}, whereas the remaining units are inactive, and
\begin{equation}
  y_{0}-y_{1}
  =\eta^2(A_{0}-A_{1}).
\end{equation}
Thus \eqref{eq:first-value-margin} gives a correct-logit margin of at least that $y_0-y_1\ge$ $5\eta^3/8$ at the center.  Allow \(c_K\) to vary over a sufficiently small fixed neighborhood of \(0\). This does not affect the attention weights, since \(q_F^\top c_K\) is added to every attention score. Together with the fixed neighborhoods chosen above for \((a,b,r,s)\), \((C_V,b_V)\), and \((W_1,b_1,W_2,b_2)\), this gives a set of normalized volume \(c_F>0\), independent of \(n\).

For the upper bound, swap the two rows of $W_2$ and the two entries of $b_2$, leaving every other coordinate unchanged.  This is a measure-preserving involution of the raw parameter cube.  It maps the exact solutions of \FIRST\ to the exact solutions of the complementary task $w\mapsto1-w_1$.  The two exact-solution sets are disjoint, so each occupies at most one half of the cube.  In particular, the NESV is $O(1)$.
\end{proof}

\subsection{\INDEX}
\label{app:index}

We use the notation from \cref{sec:index-components}.  Thus the attention score at position $j$ is
\begin{equation}
  g_j=(ak+b)^\top(rj+sj^2+c_K),
\end{equation}
and
\begin{equation}
  \chi=a^\top s,
  \qquad
  \ell=2b^\top s+a^\top r,
  \qquad
  \delta=b^\top s,
  \qquad
  \phi=b^\top r.
  \label{eq:app-index-coefficients}
\end{equation}
For $k+\Delta\in[n]$, the \emph{score gap} at position $k+\Delta$ is
\begin{equation}
  g_{k+\Delta}-g_k
  =
  \Delta C_{\mathrm{lin}}(k)
  +\Delta^2C_{\mathrm{quad}}(k),
  \label{eq:app-index-displacement}
\end{equation}
where
\begin{equation}
  C_{\mathrm{lin}}(k)=2\chi k^2+\ell k+\phi,
  \qquad
  C_{\mathrm{quad}}(k)=\chi k+\delta.
  \label{eq:app-index-drift-polynomials}
\end{equation}

\subsubsection{Lower bound}
\label{sec:index_lower}

Choose $\eta=\min\{B/16,1/16\}$ and consider the solution (which we call the \emph{center}):
\begin{equation}
  a^0=(2\eta,0),
  \qquad
  b^0=(0,-\eta),
  \qquad
  r^0=(\eta,0),
  \qquad
  s^0=(0,\eta).
  \label{eq:app-index-center}
\end{equation}
At this point,
\begin{equation}
  \chi=\ell=\phi=0,
  \qquad
  \delta=-\eta^2,
\end{equation}
and hence
\begin{equation}
  g_k-g_j=\eta^2(j-k)^2.
  \label{eq:app-index-center-gap}
\end{equation}
To obtain a positive-volume family around this point, consider the local coordinate map
\begin{equation}
  \Psi(a,b,r,s)
  =
  (\chi,\ell,a_1,a_2,b_1,b_2,r_2,s_2).
  \label{eq:index-local-coordinate-map}
\end{equation}
We call $(a,b,r,s)$ the untransformed coordinates and $(\chi,\ell,a_1,a_2,b_1,b_2,r_2,s_2)$ the transformed coordinates. For fixed $(a,b,r_2,s_2)$, the Jacobian of $(s_1,r_1)\mapsto(\chi,\ell)$ is
\begin{equation}
  \begin{bmatrix}
    a_1 & 0\\
    2b_1 & a_1
  \end{bmatrix},
  \qquad
  \det=a_1^2.
  \label{eq:index-local-jacobian}
\end{equation}
The determinant is nonzero at \eqref{eq:app-index-center}.  The inverse function theorem therefore gives a fixed neighborhood on which $\Psi$ is one-to-one and the Jacobians of $\Psi$ and $\Psi^{-1}$ are bounded.  Hence there are constants \(c,C>0\), independent of \(n\), such that
every measurable set \(E\) in the transformed-coordinate neighborhood
satisfies $c\operatorname{vol}(E)
\le
\operatorname{vol}\!\left(\Psi^{-1}(E)\right)
\le
C\operatorname{vol}(E).$

Restrict
\begin{equation}
  |\chi|\le \frac{\eta^2}{64n^2},
  \qquad
  |\ell|\le \frac{\eta^2}{64n},
  \label{eq:app-index-shrinking-box}
\end{equation}
while restricting the remaining transformed coordinates to fixed-width intervals independent of $n$ sufficiently close to the center that
\begin{equation}
  \delta\le-\frac{3\eta^2}{4},
  \qquad
  |\phi|\le\frac{\eta^2}{16}.
  \label{eq:app-index-fixed-conditions}
\end{equation}
For every $k\le n$, these conditions imply
\begin{equation}
  C_{\mathrm{quad}}(k)
  \le -\frac{\eta^2}{2},
  \qquad
  |C_{\mathrm{lin}}(k)|
  \le \frac{\eta^2}{4}.
  \label{eq:app-index-uniform-drift-control}
\end{equation}
Indeed, the contribution of $\chi k$ to $C_{\mathrm{quad}}(k)$ is at most $\eta^2/64$, and the three contributions to $C_{\mathrm{lin}}(k)$ are bounded by $2\eta^2/64$, $\eta^2/64$, and $\eta^2/16$, respectively.  Therefore, for every nonzero integer $\Delta$ with $k+\Delta\in[n]$,
\begin{align}
  g_k-g_{k+\Delta}
  &=-\Delta C_{\mathrm{lin}}(k)
    -\Delta^2C_{\mathrm{quad}}(k)
  &\ge \frac{\eta^2}{4}\Delta^2.
  \label{eq:app-index-robust-gap}
\end{align}
Under the $\log n$ attention scale, the total attention outside the indexed position is consequently at most
\begin{equation}
  2\sum_{d=1}^{\infty}n^{-(\eta^2/4)d^2},
  \label{eq:app-index-off-target-mass}
\end{equation}
which tends to zero as $n\to\infty$ uniformly in the index \(k\) and over all positional parameters \((a,b,r,s)\) in the local neighborhood above satisfying
\eqref{eq:app-index-shrinking-box} and
\eqref{eq:app-index-fixed-conditions}.

Set \(C_V^0=\eta I_2\), \(b_V^0=0\), and take the parameter values
\[
W_1^0=
\begin{bmatrix}
\eta I_2\\
0_{(h-2)\times 2}
\end{bmatrix},
\qquad
b_1^0=(\eta^2,\eta^2,-\eta,\ldots,-\eta)^\top,
\]
\[
W_2^0=
\begin{bmatrix}
\eta I_2 & 0_{2\times(h-2)}
\end{bmatrix},
\qquad
b_2^0=0.
\] For all sufficiently large $n$, \eqref{eq:app-index-off-target-mass} makes the indexed position receive at least $3/4$ of the total attention.  If \(w_k=\sigma\), the \(\sigma\) coordinate of the attention output exceeds the other coordinate by at least \(\eta/2\), even if every other token has symbol \(1-\sigma\). With the chosen parameters, the first two hidden units are active and, for \(c\in\{0,1\}\), the \(c\)-th unit outputs $\eta A_c+\eta^2$. The corresponding row of \(W_2\) multiplies this quantity by \(\eta\), so $y_c=\eta^2A_c+\eta^3$. The common term \(\eta^3\) cancels between the two logits, and therefore $y_\sigma-y_{1-\sigma} = \eta^2\bigl(A_\sigma-A_{1-\sigma}\bigr) \ge \frac{\eta^3}{2}$. This margin is uniform over all inputs and indexes, so there exists fixed neighborhoods independent of $n$ of \(C_V,b_V,W_1,b_1,W_2,b_2\) that produce exact solutions. The key bias $c_K$ may also vary over a fixed neighborhood because it cancels from the attention weights.

The two intervals in \eqref{eq:app-index-shrinking-box} contribute normalized widths of orders $n^{-2}$ and $n^{-1}$, whereas all other coordinates vary over intervals of width independent of $n$. Therefore
\begin{equation}
  V_B^{\INDEX,\Apos(h)}(n)=\Omega_{B,h}(n^{-3}).
  \label{eq:app-index-lower-bound}
\end{equation}

\subsubsection{Upper bound}
\label{sec:index_upper}
For each indexed position $k$, compare a constant word with a word that differs from it only at position $k$. These two inputs must be assigned different labels. 
We consider two cases. Let $m = y_{w_k}(w,k)-y_{1-w_k}(w,k)$. In any correct solution, we will have $m > 0$,  but if $m \le n^{-5}$, we call the solution \emph{small-margin}; otherwise, we call it \emph{large-margin}.

Take any one of the inequalities $y_{w_k}(w,k)>y_{1-w_k}(w,k)$; there are $4n$ such inequalities in total.  Fix arbitrary values of $a,b,r,s,c_K,C_V,b_V,W_1,b_1,W_2$, and $b_{2,1-w_k}$. As $b_{2,w_k}$ varies, the margin of the inequality varies with slope one. To satisfy $0<y_{w_k}(w,k)-y_{1-w_k}(w,k)\le n^{-5}$, $b_{2,w_k}$ needs to lie in an interval of length $n^{-5}$. Thus the set of small-margin solutions has normalized volume at most $4n\cdot n^{-5}/(2B)=O(n^{-4})$.

For every large-margin solution, changing only the symbol at position $k$ from $\tau$ to $\sigma\ne\tau$ must change the predicted label from $\tau$ to $\sigma$. By the definition of a large-margin solution, for each of the two words $w$ in this comparison, $y_{w_k}(w,k)-y_{1-w_k}(w,k)>n^{-5}$. Since the $y_{\sigma}(w,k)-y_{\tau}(w,k)$, viewed as a function of the attention output $A(w,k)$ is Lipschitz, this is impossible if the attention assigned to position $k$ is too small. 
To see this, let $w$ be the constant-$\tau$ word and let $w'$ be the word obtained by changing only position $k$ to $\sigma$. Because the query and key maps read only positional coordinates, the two words have the same attention weights. Their attention outputs therefore satisfy
\begin{equation}
  A(w',k)-A(w,k)
  =
  \alpha_k C_V(e_\sigma-e_\tau),
\end{equation}
so
\begin{equation}
  \|A(w',k)-A(w,k)\|_\infty
  \le 2B\alpha_k.
  \label{eq:app-index-output-change}
\end{equation}
Moreover, ReLU is $1$-Lipschitz, each row of $W_1$ has $\ell_1$-norm at most $2B$, and the difference between rows $\sigma$ and $\tau$ of $W_2$ has $\ell_1$-norm at most $2hB$. Hence
\begin{align}
  &\left|
  \bigl(y_\sigma(w',k)-y_\tau(w',k)\bigr)
  -\bigl(y_\sigma(w,k)-y_\tau(w,k)\bigr)
  \right|
  \notag\\
  &\hspace{5em}\le
  4hB^2\|A(w',k)-A(w,k)\|_\infty
  \le 8hB^3\alpha_k.
\end{align}
The first difference is greater than $n^{-5}$ and the second is less than $-n^{-5}$, so
\begin{equation}
  \alpha_k
  \ge
  \frac{n^{-5}}{4hB^3}.
  \label{eq:app-index-attention-lower}
\end{equation}
Hence the indexed position receives at least an inverse polynomial amount of attention. The softmax formula in \eqref{eq:attention} then implies that $g_k$ lies within a constant of the maximum score. Indeed,
\begin{equation}
  \frac{n^{-5}}{4hB^3}
  \le
  \alpha_k
  =
  \frac{n^{g_k}}{\sum_{j=1}^n n^{g_j}}
  \le
  n^{g_k-\max_j g_j},
\end{equation}
and therefore
\begin{equation}
  \max_j g_j-g_k
  \le
  5+\frac{\log(4hB^3)}{\log n}
  \le M_0
  \label{eq:app-index-near-maximum}
\end{equation}
for all sufficiently large $n$, where $M_0$ is independent of $n$.
For every interior index $2\le k\le n-1$, comparing $g_k$ with the two neighboring scores in \eqref{eq:index-score-displacement} gives
\begin{equation}
  g_{k+1}-g_k
  =
  C_{\mathrm{lin}}(k)+C_{\mathrm{quad}}(k)
  \le M_0
\end{equation}
and
\begin{equation}
  g_{k-1}-g_k
  =
  -C_{\mathrm{lin}}(k)+C_{\mathrm{quad}}(k)
  \le M_0.
\end{equation}
Thus
\begin{equation}
  C_{\mathrm{quad}}(k)+|C_{\mathrm{lin}}(k)|
  \le M_0,
  \label{eq:app-index-neighbor-bound}
\end{equation}
where $M_0$ is independent of $n$.
Since every coordinate of $a,b,s$ lies in $[-B,B]$, both $\chi=a^\top s$ and $\delta=b^\top s$ are bounded independently of $n$. Hence, for $k\le n$, $|C_{\mathrm{quad}}(k)| = |\chi k+\delta| = O(n)$, and thus $|C_{\mathrm{lin}}(k)| \le M_0+|C_{\mathrm{quad}}(k)| = O(n).$ Let $M=\lfloor n/4\rfloor$. Since $C_{\mathrm{lin}}(k)=2\chi k^2+\ell k+\phi$, we have $C_{\mathrm{lin}}(3M)-2C_{\mathrm{lin}}(2M)+C_{\mathrm{lin}}(M)=4\chi M^2$. Because each term on the left is $O(n)$ and $M=\Theta(n)$, this gives $|\chi|=O(n^{-1})$. Moreover, $C_{\mathrm{lin}}(2M)-C_{\mathrm{lin}}(M)=6\chi M^2+\ell M$. Using $|\chi|=O(n^{-1})$ in this identity gives $|\ell|=O(1)$.
These preliminary bounds make
\begin{equation}
  |C_{\mathrm{quad}}(k)|
  =
  |\chi k+\delta|
  =
  O(1)
\end{equation}
uniformly for $k\le n$. Equation \eqref{eq:app-index-neighbor-bound} then gives $|C_{\mathrm{lin}}(k)|=O(1)$. Repeating the same argument, now using these improved bounds, gives
\begin{equation}
  |\chi|=O(n^{-2}),
  \qquad
  |\ell|=O(n^{-1}).
  \label{eq:app-index-sharp-coefficient-constraints}
\end{equation}
We used the fact $\Pr\left(
 |a^\top s|\le\varepsilon_1,
 \ |2b^\top s+a^\top r|\le\varepsilon_2
\right)=O_B(\varepsilon_1\varepsilon_2)$ whenever $0<\varepsilon_1\le\varepsilon_2$ are sufficiently small. 
\begin{lemma}[Joint two-slab estimate]
\label[lemma]{lem:index-joint-slab}
Let $a,b,r,s$ be independent and uniform on $[-B,B]^2$.  There are constants $C_{\mathrm{slab}},\varepsilon_B>0$, depending only on $B$, such that, whenever $0<\varepsilon_1\le\varepsilon_2\le\varepsilon_B$,
\begin{equation}
  \Pr\!\left(
    |a^\top s|\le\varepsilon_1,
    \ |2b^\top s+a^\top r|\le\varepsilon_2
  \right)
  \le
  C_{\mathrm{slab}}\varepsilon_1\varepsilon_2.
  \label{eq:app-joint-slab}
\end{equation}
\end{lemma}

\begin{proof}
Scaling changes only the constant, so take $B=1$.  Fix $(a,s)$, and set
\begin{equation}
  \varrho=(\|a\|_2^2+\|s\|_2^2)^{1/2}.
\end{equation}
The second constraint defines a slab in $(b,r)$ with normal vector $(2s,a)$. One coordinate of this vector has magnitude at least a constant multiple of $\varrho$, so, after fixing the other three coordinates, the allowed fraction of the remaining coordinate is at most 
\begin{equation}
  C\min\left\{1,\frac{\varepsilon_2}{\varrho}\right\}.
  \label{eq:joint-second-fraction}
\end{equation}
To estimate the first constraint at fixed radius $\varrho$. Write
\begin{equation}
  a=\varrho\cos\vartheta\,\widehat a,
  \qquad
  s=\varrho\sin\vartheta\,\widehat s,
  \qquad
  0\le\vartheta\le\frac{\pi}{2},
  \label{eq:joint-angular-parameterization}
\end{equation}
where $\widehat a,\widehat s\in\mathbb S^1$.  If $\varphi$ is the angle between $\widehat a$ and $\widehat s$, then
\begin{equation}
  |a^\top s|
  =
  \varrho^2\sin\vartheta\cos\vartheta\,|\cos\varphi|.
\end{equation}
For fixed $\vartheta$, the fraction of angle $\varphi$ satisfying the first constraint is at most
\begin{equation}
  C\min\left\{
    1,
    \frac{\varepsilon_1}
         {\varrho^2\sin\vartheta\cos\vartheta}
  \right\}.
\end{equation}
The angular density contains the factor $\sin\vartheta\cos\vartheta$, therefore integrating over $\vartheta$ bounds the angular fraction by
\begin{equation}
  C\min\left\{1,\frac{\varepsilon_1}{\varrho^2}\right\}.
  \label{eq:joint-first-fraction}
\end{equation}
because when $t\le1$, $\int_0^{\pi/2}\min\{\sin\vartheta\cos\vartheta,t\}\,d\vartheta =O(t)$.
The radial volume element for $(a,s)\in\mathbb R^4$is proportional to $\varrho^3d\varrho$.  Enlarging the cube to a fixed ball and combining \eqref{eq:joint-second-fraction} and \eqref{eq:joint-first-fraction}, the required probability is at most a constant times
\begin{equation}
  \int_0^{R_0}
  \varrho^3
  \min\left\{1,\frac{\varepsilon_1}{\varrho^2}\right\}
  \min\left\{1,\frac{\varepsilon_2}{\varrho}\right\}
  d\varrho.
  \label{eq:joint-radial-integral}
\end{equation}
Assume $0<\varepsilon_1\le\varepsilon_2\le1$.  If $\varepsilon_2\le\sqrt{\varepsilon_1}$, split the integral at $\varepsilon_2$ and $\sqrt{\varepsilon_1}$ to obtain
\begin{align}
  I
  &\le
  \int_0^{\varepsilon_2}\varrho^3\,d\varrho
  +\varepsilon_2
   \int_{\varepsilon_2}^{\sqrt{\varepsilon_1}}\varrho^2\,d\varrho
  +\varepsilon_1\varepsilon_2
   \int_{\sqrt{\varepsilon_1}}^{R_0}d\varrho
  \notag\\
  &\le C\varepsilon_1\varepsilon_2.
\end{align}
If $\sqrt{\varepsilon_1}\le\varepsilon_2$, split instead at $\sqrt{\varepsilon_1}$ and $\varepsilon_2$.
Rescaling restores $B$ and proves the lemma.
\end{proof}
Applying the above lemma to \eqref{eq:app-index-sharp-coefficient-constraints}, the large-margin solutions have normalized volume $O(n^{-3})$: the entries of $a,b,r,s$ are independent parameter coordinates, and all remaining coordinates integrate out. Adding the $O(n^{-4})$ small-margin volume gives
\begin{equation}
  V_B^{\INDEX,\Apos(h)}(n)
  =
  O_{B,h}(n^{-3}).
\end{equation}
Together with \eqref{eq:app-index-lower-bound}, this proves the \INDEX\ claim in \cref{thm:nesv-bounds}.

\subsection{\MAJORITY}
\label{app:majority}

In the position-free model $\Acount$, every input with the same number of ones produces the same attention output.  The proof therefore reduces \MAJORITY\ to a one-dimensional sequence indexed by the count and measures how much of the parameter space places the decision boundary between the two central counts.
Let
\begin{equation}
  t=\sum_{j=1}^n w_j,
  \qquad
  v_b=V(e_b),
  \qquad
  q=Q(0_2),
\end{equation}
and define the query--key score contrast
\begin{equation}
  X=q^\top W_K(e_1-e_0),
  \qquad
  \Gamma=\exp((\log n)X).
  \label{eq:majority-score-contrast}
\end{equation}
All zero tokens share one score and all one tokens share another, so the total attention mass on the one tokens and the attention output are exactly
\begin{equation}
  \rho_t=\frac{t\Gamma}{n-t+t\Gamma},
  \qquad
  A_t=v_0+\rho_t(v_1-v_0).
  \label{eq:majority-count-path}
\end{equation}
Thus the $2^n$ inputs reduce to $n+1$ possible count levels. Throughout this subsection, write $n=2m+1$.
For an attention scale $\beta>0$ and a score contrast $x$, let 
\begin{equation}
  \Delta\rho_{n,\beta}(x)=\rho_{m+1}-\rho_m,
\end{equation}
where $\Gamma=e^{\beta x}$.  The counts $m$ and $m+1$ are the two adjacent levels at which the majority label changes.

\begin{lemma}[Central attention gap]
\label[lemma]{lem:central-attention-gap}
For $\Gamma=e^{\beta x}$,
\begin{align}
  \Delta\rho_{n,\beta}(x)
  &=
  \frac{ne^{\beta x}}
  {(m+1+me^{\beta x})(m+(m+1)e^{\beta x})}
  \label{eq:central-gap-rational}
  \\
  &=
  \frac{n}
  {n^2+(n^2-1)\sinh^2(\beta x/2)}.
  \label{eq:central-gap-hyperbolic}
\end{align}
Moreover,
\begin{equation}
  0<\Delta\rho_{n,\beta}(x)\le\frac1n,
  \qquad
  \int_{\mathbb R}\Delta\rho_{n,\beta}(x)\,dx
  =\frac{4\operatorname{artanh}(1/n)}{\beta}.
  \label{eq:central-gap-integral}
\end{equation}
\end{lemma}

\begin{proof}
Substitute $t=m,m+1$ into \eqref{eq:majority-count-path} and subtract to obtain \eqref{eq:central-gap-rational}.  Dividing its denominator by $e^{\beta x}$ and using $n=2m+1$ gives \eqref{eq:central-gap-hyperbolic}.  The maximum is attained at $x=0$. For the integral, set $z=\tanh(\beta x/2)$.  Then
\begin{equation}
  dx=\frac{2\,dz}{\beta(1-z^2)},
  \qquad
  \Delta\rho_{n,\beta}(x)=\frac{n(1-z^2)}{n^2-z^2}.
\end{equation}
Therefore
\begin{equation}
  \int_{\mathbb R}\Delta\rho_{n,\beta}(x)\,dx
  =\frac{2n}{\beta}\int_{-1}^1\frac{dz}{n^2-z^2}
  =\frac{4\operatorname{artanh}(1/n)}{\beta}.
\end{equation}
\end{proof}

For fixed values of all parameters except $b_{2,1}$, exact correctness requires $b_{2,1}$ to place the decision threshold between the two central count levels.  Define the bias-free logit difference
\begin{equation}
  H(z)
  =(e_1-e_0)^\top W_2\relu(W_1z+b_1).
\end{equation}

\begin{lemma}[Output-bias fiber]
\label[lemma]{lem:majority-bias-fiber}
Condition on every raw parameter except $b_{2,1}$.  The set of values of $b_{2,1}$ that can extend the conditioned coordinates to an exact \MAJORITY\ solution has length at most
\begin{equation}
  8hB^3\Delta\rho_{n,\beta}(X).
  \label{eq:majority-bias-fiber-length}
\end{equation}
\end{lemma}

\begin{proof}
Correctness on the count-$m$ and count-$(m+1)$ levels requires
\begin{align}
  H(A_m)+b_{2,1}-b_{2,0}&<0,\\
  H(A_{m+1})+b_{2,1}-b_{2,0}&>0.
\end{align}
Thus $b_{2,1}$ must lie in
\begin{equation}
  \bigl(b_{2,0}-H(A_{m+1}),\ b_{2,0}-H(A_m)\bigr),
\end{equation}
whose length is $H(A_{m+1})-H(A_m)$ when this difference is positive. ReLU is $1$-Lipschitz, each row of $W_1$ has $\ell_1$-norm at most $2B$, and the difference between the two rows of $W_2$ has $\ell_1$-norm at most $2hB$. Hence $H$ is $4hB^2$-Lipschitz in $\ell_\infty$. Moreover, $A_{m+1}-A_m =\Delta\rho_{n,\log n}(X)(v_1-v_0)$, so
\begin{equation}
  \|A_{m+1}-A_m\|_\infty
  \le 2B \Delta\rho_{n,\beta}(X).
\end{equation}
Multiplying the two bounds proves the claim.
\end{proof}

To average the fiber length over the query and key parameters, we use a uniform density bound for $X$.

\begin{lemma}[Density of the score contrast]
\label[lemma]{lem:majority-contrast-density}
Under the uniform raw-cube measure, the random variable $X=q^\top W_K(e_1-e_0)$ has a density $f_X$ satisfying
\begin{equation}
  \|f_X\|_\infty\le B^{-2}.
\end{equation}
\end{lemma}

\begin{proof}
Write $\Delta_K=W_K(e_1-e_0)$.  Each coordinate of $\Delta_K$ is the difference of two independent uniform variables and has triangular density bounded by $(2B)^{-1}$.  Conditional on $q_{\mathrm{agg}}\ne0$, choose a coordinate with magnitude $\|q_{\mathrm{agg}}\|_\infty$ and condition on the other coordinate of $\Delta_K$.  Scaling gives
\begin{equation}
  \|f_{X\mid q_{\mathrm{agg}}}\|_\infty
  \le
  \frac{1}{2B\|q_{\mathrm{agg}}\|_\infty}.
\end{equation}
Since $q_{\mathrm{agg}}=Q(0)=b_Q$ is uniform on $[-B,B]^2$,
\begin{equation}
  \mathbb E\bigl[\|q_{\mathrm{agg}}\|_\infty^{-1}\bigr]
  =\frac{2}{B}.
\end{equation}
Averaging proves the result.
\end{proof}

By Fubini's theorem and \cref{lem:central-attention-gap,lem:majority-bias-fiber,lem:majority-contrast-density}, \begin{align}
  V_B^{\MAJORITY,\Acount(h)}(n)
  &\le
  4hB^2\,\mathbb E\bigl[\Delta\rho_{n,\log n}(X)\bigr]
  \notag\\
  &\le
  4hB^2\|f_X\|_\infty
  \int_{\mathbb R}\Delta\rho_{n,\log n}(x)\,dx
  \notag\\
  &\le
  \frac{16h\operatorname{artanh}(1/n)}{\log n}.
  \label{eq:majority-upper-bound}
\end{align}
Since $\operatorname{artanh}(1/n)\le n/(n^2-1)$, this is $O_h((n\log n)^{-1})$.

For the matching lower bound, we need a positive mass of score contrasts on which the central gap is of order $1/n$ and a fixed-volume set of the remaining parameters that turns this gap into a correct bias interval.

\begin{lemma}[Lower contrast mass]
\label[lemma]{lem:majority-lower-contrast}
There is a constant $c_X=c_X(B)>0$ such that, for all sufficiently large $n$,
\begin{equation}
  \Pr\!\left(|X|\le\frac{1}{\log n}\right)
  \ge\frac{c_X}{\log n}.
  \label{eq:majority-lower-contrast-event}
\end{equation}
\end{lemma}

\begin{proof}
Write $\Delta_K=W_K(e_1-e_0)$ and $q_{\mathrm{agg}}=(q_1,q_2)$.  Restrict to
\begin{equation}
  q_1\in[B/2,B],
  \qquad
  |q_2|\le B/4,
  \qquad
  |\Delta_{K,2}|\le B/2.
\end{equation}
These independent events have probabilities at least $1/4$, $1/4$, and $3/8$.  Conditional on them, the inequality
\begin{equation}
  |q_1\Delta_{K,1}
   +q_2\Delta_{K,2}|
  \le\varepsilon
\end{equation}
asks that $\Delta_{K,1}$ lie in an interval of length at least $2\varepsilon/B$, centered in $[-B/4,B/4]$.  If $\varepsilon\le B^2/8$, the interval lies in $[-B/2,B/2]$, where the triangular density of $\Delta_{K,1}$ is at least $3/(8B)$.  Hence
\begin{equation}
  \Pr(|X|\le\varepsilon)
  \ge\frac{9\varepsilon}{512B^2}.
\end{equation}
Taking $\varepsilon=1/\log n$ proves the lemma.
\end{proof}

\begin{lemma}[Positive-volume monotone decoder family]
\label[lemma]{lem:majority-monotone-decoder}
There are constants $c_\Omega,\kappa>0$ and a set of values of $(W_V,b_V,W_1,b_1,W_2,b_{2,0})$ of normalized volume at least $c_\Omega$ such that the bias-free logit difference along $v_0+p(v_1-v_0)$ is affine with derivative at least $\kappa$ for all $p\in[0,1]$. For every parameter choice in this set, $b_{2,1}$ may vary over an interval of length at least
\begin{equation}
  \kappa(\rho_{m+1}-\rho_m)
\end{equation}
while preserving exact \MAJORITY\ correctness.
\end{lemma}

\begin{proof}
Let $e_V=(1,0)^\top$ and choose $W_V,b_V$ so that
\begin{equation}
  v_0=0,
  \qquad
  v_1=\eta e_V.
\end{equation}
Give one hidden unit first-layer row $\eta e_V^\top$ and bias $\eta$, make every other hidden unit strictly inactive, and set the difference between output rows on the active hidden unit to $\eta$.  Set the other output-row differences and $b_{2,0}$ to zero.  Along the value segment, define
\begin{equation}
  \psi(p)=H(v_0+p(v_1-v_0)).
\end{equation}
For these parameter values,
\begin{equation}
  \psi(p)=\eta^2+\eta^3p.
\end{equation}
The active pre-activation lies in $[\eta,\eta+\eta^2]$, while every other pre-activation equals $-\eta$. A fixed open raw-coordinate box around these parameter values therefore preserves the activation pattern, keeps $\psi$ affine with derivative at least $\eta^3/2$, and keeps $b_{2,0}-\psi$ inside $(-B,B)$. Take $\kappa=\eta^3/2$ and let $c_\Omega$ be the normalized volume of this box. For parameters in the box, the interval
\begin{equation}
  \bigl(
    b_{2,0}-\psi(\rho_{m+1}),
    \ b_{2,0}-\psi(\rho_m)
  \bigr)
\end{equation}
has length at least $\kappa(\rho_{m+1}-\rho_m)$ and lies in the raw bias range.  Since both $\rho_t$ and $\psi$ are increasing, every $t\le m$ gives a negative logit difference and every $t\ge m+1$ gives a positive one. Thus all count levels are classified correctly.
\end{proof}
On the event in \eqref{eq:majority-lower-contrast-event}, $|\log n\,X|\le1$, so  \eqref{eq:central-gap-hyperbolic} gives $\rho_{m+1}-\rho_m
  \ge
  \frac{1}{n\bigl(1+\sinh^2(1/2)\bigr)}
  =\Omega(n^{-1})$. Restrict the $(W_V,b_V,W_1,b_1,W_2,b_{2,0})$ to the set in \cref{lem:majority-monotone-decoder}.  The remaining bias interval then has normalized length $\Omega(1/n)$.  Multiplying this length by the $\Omega(1/\log n)$ contrast event and the fixed parameter box volume gives
\begin{equation}
  V_B^{\MAJORITY,\Acount(h)}(n)
  \ge \frac{c_-}{n\log n}.
  \label{eq:majority-lower-bound}
\end{equation}
Together with \eqref{eq:majority-upper-bound}, this proves the \MAJORITY\ claim in \cref{thm:nesv-bounds}.

\subsection{\PARITY}
\label{app:parity}

For \PARITY, the count reduction in \eqref{eq:majority-count-path} gives a sequence of points along the value segment on which the required label flips.  The proof compares the resulting number of sign changes with the number of activation changes available to a one-hidden-layer ReLU classifier.
If $v_0\ne v_1$, then
\begin{equation}
  0=\rho_0<\rho_1<\cdots<\rho_n=1.
\end{equation}
Since every count-$t$ attention output has the form $A_t=v_0+\rho_t(v_1-v_0),$ all such outputs lie on the line segment joining $v_0$ and $v_1$.
Parameterize this segment by
$A(p)=v_0+p(v_1-v_0)$, $p\in[0,1]$ and define the difference between the two output logits along it by
\begin{equation}
  f(p)
  =
  (e_1-e_0)^\top
  \bigl[
    W_2\relu(W_1A(p)+b_1)+b_2
  \bigr].
  \label{eq:parity-segment-function}
\end{equation}This is a continuous piecewise-affine function of $p$, and each hidden preactivation can change sign at most once.

\begin{proof}[Proof of the \PARITY\ bound in \cref{thm:nesv-bounds}]
If $v_0=v_1$, every count level has the same attention output, so exact parity is impossible.  Otherwise, exact parity makes $f(\rho_t)$ alternate in sign.  By continuity, for each $t=0,\ldots,n-1$, choose a zero $z_t\in(\rho_t,\rho_{t+1})$.
For each $t=1,\ldots,n-1$, the interval $(z_{t-1},z_t)$ contains $\rho_t$, where $f$ is nonzero.  If no hidden pre-activation changed sign in this interval, the activation pattern would be fixed and $f$ would be affine there.  An affine function that vanishes at both endpoints is identically zero, contradicting $f(\rho_t)\ne0$.  Hence each of the $n-1$ disjoint intervals $(z_{t-1},z_t)$ contains a point at which a hidden ReLU pre-activation changes sign.  Since each hidden pre-activation changes sign at most once, this requires $h\ge n-1$.  Therefore the exact-solution set is empty whenever $n>h+1$.
\end{proof}

\section{Proof of the $\Anorm(h)$ \INDEX\ Bound}
\label{app:norm}

This appendix proves \cref{thm:index-norm}.  Recall the fixed map
\begin{equation}
  F(v)=\frac{v}{\sqrt{1+\|v\|_2^2}}
\end{equation}
and the $\Anorm(h)$ positional score
\begin{equation}
  g_j^{\mathrm{norm}}
  =(ak+b)^\top(rj+c_K)
  +\bigl(b+F(ak)\bigr)^\top sj^2.
  \label{eq:app-norm-score}
\end{equation}
\begin{equation}
  \widehat\chi=\widehat a^\top s,
  \qquad
  \widehat\ell=\ell+2\widehat\chi,
  \qquad
  \varepsilon_\chi(k)
  =\bigl(F(ak)-\widehat a\bigr)^\top s,
  \label{eq:app-norm-effective-terms}
\end{equation}
with $\delta=b^\top s$ and $\phi=b^\top r$ as in \eqref{eq:index-effective-coefficients}. Then,
\begin{equation}
  C_{\mathrm{quad}}^{\mathrm{norm}}(k)
=\widehat\chi+\delta+\varepsilon_\chi(k),
  \qquad
  C_{\mathrm{lin}}^{\mathrm{norm}}(k)
  =\widehat\ell k+\phi+2k\varepsilon_\chi(k),
  \label{eq:app-norm-gap-coefficients}
\end{equation}
Whenever $k+\Delta\in[n]$,
\begin{equation}
  g_{k+\Delta}^{\mathrm{norm}}-g_k^{\mathrm{norm}}
  =
  \Delta C_{\mathrm{lin}}^{\mathrm{norm}}(k)
  +\Delta^2 C_{\mathrm{quad}}^{\mathrm{norm}}(k).
  \label{eq:app-norm-displacement}
\end{equation}
Since $F(ak)$ is collinear with $a$,
\begin{equation}
  F(ak)
  =
  \lambda_a(k)\widehat a,
  \qquad
  \lambda_a(k)
  =
  \frac{k\|a\|_2}{\sqrt{1+k^2\|a\|_2^2}}.
  \label{eq:norm-lambda}
\end{equation}
Moreover,
\begin{equation}
  0\le1-\lambda_a(k)
  \le\frac{1}{2k^2\|a\|_2^2},
  \qquad
  |2k\varepsilon_\chi(k)|
  \le\frac{\|s\|_2}{k\|a\|_2^2}.
  \label{eq:app-norm-remainder}
\end{equation}

\subsection{Lower bound}
\label{app:norm-lower}

Use the same positional center as in \eqref{eq:app-index-center}.  At this point, $\widehat a^0=(1,0)$ is orthogonal to $s^0=(0,\eta)$, and hence
\begin{equation}
  \widehat\ell=\phi=0,
  \qquad
  Q_{a^0}(k)=-\eta^2
  \qquad (k\ge1).
\end{equation}
Thus \begin{equation}
  C_{\mathrm{quad}}^{\mathrm{norm}}(k)=-\eta^2,
  \qquad
  C_{\mathrm{lin}}^{\mathrm{norm}}(k)=0,
\end{equation}
and therefore,
\begin{equation}
  g_k^{\mathrm{norm}}-g_j^{\mathrm{norm}}
  =\eta^2(j-k)^2.
  \label{eq:app-norm-center-gap}
\end{equation}

Work in a fixed neighborhood of this center on which $a_1>\eta$. The map $a\mapsto\widehat a$ is smooth there.  Replace the raw coordinate $r_1$ by $\widehat\ell$, keeping $(a,b,r_2,s)$ as the other coordinates.  Since
\begin{equation}
  \frac{\partial\widehat\ell}{\partial r_1}=a_1,
\end{equation}
the corresponding coordinate change is locally invertible, with Jacobian and inverse Jacobian bounded on a sufficiently small fixed neighborhood. Thus boxes in these transformed coordinates and their preimages in the original coordinates have volumes within fixed multiplicative constants, independent of $n$.
Shrink the fixed intervals for the other transformed coordinates so that throughout the neighborhood
\begin{equation}
  \|a\|_2\ge\eta,
  \qquad
  b^\top s\le-\frac{7\eta^2}{8},
  \qquad
  |\widehat a^\top s|
  \le
  \min\left\{\frac{\eta^2}{8},\frac{\eta^4}{16}\right\},
  \qquad
  |\phi|\le\frac{\eta^2}{16}.
  \label{eq:norm-lower-fixed-conditions}
\end{equation}
By \eqref{eq:norm-lambda}, for every $k\ge1$,
\begin{equation}
  C_{\mathrm{quad}}^{\mathrm{norm}}(k)
  =\delta+\lambda_a(k)\widehat\chi
  \le-\frac{3\eta^2}{4}.
  \label{eq:norm-lower-curvature}
\end{equation}
Furthermore, \eqref{eq:app-norm-remainder} and \eqref{eq:norm-lower-fixed-conditions} give
\begin{equation}
  |2k\varepsilon_\chi(k)|
  \le\frac{|\widehat a^\top s|}{\|a\|_2^2}
  \le\frac{\eta^2}{16}
  \qquad (k\ge1).
  \label{eq:norm-lower-remainder}
\end{equation}
Here we used that $2k\varepsilon_\chi(k)=2k(\lambda_a(k)-1)\widehat a^\top s$.

Now restrict the only shrinking transformed coordinate by
\begin{equation}
  |\widehat\ell|\le\frac{\eta^2}{16n}.
  \label{eq:norm-lower-zeta-width}
\end{equation}
For every $k\le n$, equations \eqref{eq:norm-lower-fixed-conditions}, \eqref{eq:norm-lower-remainder}, and \eqref{eq:norm-lower-zeta-width} imply
\begin{equation}
  |\widehat\ell k+\phi+2k\varepsilon_\chi(k)|
  \le\frac{3\eta^2}{16}.
\end{equation}
Therefore, for every nonzero integer $\Delta$ with $k+\Delta\in[n]$,
\begin{align}
  g_k^{\mathrm{norm}}
  -g_{k+\Delta}^{\mathrm{norm}}
  &\ge
  \frac{3\eta^2}{4}\Delta^2
  -\frac{3\eta^2}{16}|\Delta|
  \notag\\
  &\ge
  \frac{\eta^2}{2}\Delta^2.
  \label{eq:norm-lower-uniform-gap}
\end{align}
Under the $\log n$ attention scale, the total attention outside position $k$ is at most
\begin{equation}
  2\sum_{d=1}^{\infty}n^{-(\eta^2/2)d^2},
\end{equation}
which tends to zero uniformly in $k$ and in the parameters in the neighborhood above. Thus the same fixed neighborhoods of $C_V,b_V,W_1,b_1,W_2,b_2$ used in the baseline \INDEX\ lower bound preserve a positive logit difference for the correct label, and $c_K$ may vary over a fixed neighborhood because it cancels from the attention weights.
The transformed coordinate $\widehat\ell$ varies over an interval of width $\Theta(n^{-1})$, whereas $a,b,r_2,s,c_K,C_V,b_V,W_1,b_1,W_2,b_2$ vary over fixed-width intervals.  Boundedness of the local Jacobian then gives
\begin{equation}
  V_B^{\INDEX,\Anorm(h)}(n)
  =\Omega_{B,h}(n^{-1}).
  \label{eq:norm-volume-lower}
\end{equation}

\subsection{Upper bound}
\label{app:norm-upper}

Use the same $4n$ inequalities $y_{w_k}(w,k)>y_{1-w_k}(w,k)$ as in the original \INDEX\ upper bound.  The same output-bias argument shows that exact solutions for which one of these margins lies in $(0,n^{-5}]$ have normalized volume $O_B(n^{-4})$. For every remaining exact solution, the indexed position receives at least $c_0n^{-5}$ attention in each constant-word/single-change comparison.  Hence its score lies within a constant $M_0$ of the maximum score, uniformly in $n$.  Comparing the indexed position with its two neighbors in \eqref{eq:app-norm-displacement} gives
\begin{equation}
  C_{\mathrm{quad}}^{\mathrm{norm}}(k)
  +|C_{\mathrm{lin}}^{\mathrm{norm}}(k)|
  \le M_0
  \qquad (2\le k\le n-1).
  \label{eq:norm-upper-neighbor-bound}
\end{equation}

Since $\|F(ak)\|_2<1$ and the parameters are bounded, there exists a constant $C_0=C_0(B)$ such that
\begin{equation}
  |C_{\mathrm{quad}}^{\mathrm{norm}}(k)|
  =|\bigl(b+F(ak)\bigr)^\top s|
  \le C_0
  \qquad (a,b,s\in[-B,B]^2,\ k\ge1).
\end{equation}
It follows from \eqref{eq:norm-upper-neighbor-bound} that
\begin{equation}
  |\widehat\ell k+\phi+2k\varepsilon_\chi(k)|\le M_1
  \qquad (2\le k\le n-1)
  \label{eq:norm-upper-linear-bound}
\end{equation}
for another constant $M_1=M_1(B,h)$.

Split the raw parameter cube according to whether
\begin{equation}
  \|a\|_2\le n^{-1/2}.
  \label{eq:norm-small-a-region}
\end{equation}
Because $a$ consists of two independent raw coordinates, the normalized volume of this part is at most the area of a radius-$n^{-1/2}$ disk divided by $(2B)^2$, and hence is $O_B(n^{-1})$. On its complement, take $k_0=\lfloor n/2\rfloor$.  For all sufficiently large $n$, $k_0$ is interior and $k_0\ge n/3$.  By \eqref{eq:app-norm-remainder},
\begin{equation}
  |2k\varepsilon_\chi(k_0)|
  \le
  \frac{\|s\|_2}{k_0\|a\|_2^2}
  =O_B(1)
\end{equation}
because $\|a\|_2>n^{-1/2}$.  Boundedness also gives $|\phi|=O_B(1)$.  Equation \eqref{eq:norm-upper-linear-bound}  at $k_0$ therefore forces
\begin{equation}
  |\widehat\ell|\le \frac{C}{n}
  \label{eq:norm-upper-zeta-slab}
\end{equation}
for $C=C(B,h)$.

Conditional on $(a,b,s)$, 
\begin{equation}
  \widehat\ell=a^\top r+
  \bigl(2b^\top s+2\widehat a^\top s\bigr),
\end{equation}
so \eqref{eq:norm-upper-zeta-slab} defines a slab in the independent variable $r$ with normal $a$.  Choose a coordinate $i$ with $|a_i|\ge\|a\|_2/\sqrt2$ and fix the other coordinate of $r$. The allowed values of $r_i$ lie in an interval of length at most $C/(n\|a\|_2)$, up to a constant depending only on $B$ and $h$. Thus the normalized volume is at most
\begin{equation}
  C_B\min\left\{1,\frac{1}{n\|a\|_2}\right\}.
  \label{eq:norm-upper-fiber}
\end{equation}
Enlarging the square for $a$ to a fixed-radius disk and integrating in polar coordinates on the complement of \eqref{eq:norm-small-a-region} gives
\begin{align}
  \int_{\substack{a\in[-B,B]^2\\
                   \|a\|_2>n^{-1/2}}}
  \frac{C_B}{n\|a\|_2}\,da
  &\le
  \frac{C_B'}{n}
  \int_{n^{-1/2}}^{\sqrt2B}d\varrho
  \notag\\
  &=O_B(n^{-1}).
  \label{eq:norm-upper-integral}
\end{align}
All remaining raw coordinates integrate out.  Adding the $O(n^{-1})$ small-$a$ region and the $O(n^{-4})$ small-margin region gives
\begin{equation}
  V_B^{\INDEX,\Anorm(h)}(n)
  =O_{B,h}(n^{-1}).
  \label{eq:norm-volume-upper}
\end{equation}
Together with \eqref{eq:norm-volume-lower}, this proves \cref{thm:index-norm}.

\section{Fixed Sinusoidal Encoding on \MAJORITY}
\label{app:sinusoidal-majority}

This appendix analyzes the fixed-frequency sinusoidal \MAJORITY\ model used in the discussion. 

\subsection{Model and matched bound}
\label{app:sinusoidal-majority-model}

Let $n=2m+1$ be odd, let $p_j=(\sin j,\cos j)^\top$, and define
\begin{equation}
  x_j^{\sin}
  =
  \begin{bmatrix}
    e_{w_j}\\
    p_j
  \end{bmatrix}
  \in\mathbb R^4,
  \qquad
  z^{\sin}=0_4.
  \label{eq:sin-majority-tokens}
\end{equation}
The frequency is fixed at one and is not a trainable coordinate. The query, key, and value maps are unrestricted affine maps from $\mathbb R^4$ to $\mathbb R^2$, the attention multiplier is $\beta_n=\log n$, and the output is the same width-$h$ one-hidden-layer ReLU classifier as in \cref{sec: setup}. We denote this architecture by $\mathcal A_{\mathrm{sin}}(h)$. All scalar parameters are bounded in $[-B, B]$.

\begin{theorem}[Fixed-sinusoidal \MAJORITY\ volume]
\label{thm:sinusoidal-majority-volume}
For every fixed $B>0$ and $h\ge2$,
\begin{equation}
V_B^{\MAJORITY,\mathcal A_{\mathrm{sin,gen}}(h)}(n)=
  \Theta(\frac{1}{n^3\log n}).
  \label{eq:sin-majority-matched-bound}
\end{equation}
\end{theorem}

Write the columns of $W_K$ and $W_V$ as
\begin{equation}
  W_K=[k_0\ k_1\ k_{\sin}\ k_{\cos}],
  \qquad
  W_V=[v_0\ v_1\ v_{\sin}\ v_{\cos}],
\end{equation}
and define
\begin{align}
  q&=b_Q,
  &
  \lambda&=\beta_n q^\top(k_1-k_0),
  &
  \Gamma&=e^\lambda,
  \notag\\
  a&=\beta_n
  \begin{bmatrix}
    q^\top k_{\sin}\\
    q^\top k_{\cos}
  \end{bmatrix},
  &
  P_j&=e^{a^\top p_j},
  &
  T&=\sum_{j=1}^nP_j,
  \notag\\
  v_b&=v_b+b_V\quad(b\in\{0,1\}),
  &
  C_V&=[v_{\sin}\ v_{\cos}],
  &
  \mu&=\frac1T\sum_{j=1}^nP_jp_j.
  \label{eq:sin-majority-reduction-parameters}
\end{align}
The common key-bias and symbol-zero score contributions cancel, so the attention weights are
\begin{equation}
  \alpha_j(w)
  =
  \frac{P_j\Gamma^{w_j}}{\sum_{i=1}^nP_i\Gamma^{w_i}}.
  \label{eq:sin-majority-reduced-attention}
\end{equation}
For $t\in[0,1]^n$, introduce the algebraic extensions
\begin{align}
  \mathsf D(t)
  &=T+(\Gamma-1)\sum_jP_jt_j,
  \label{eq:sin-majority-denominator}\\
  \mathsf N(t)
  &=\sum_jP_j(v_0+C_Vp_j)
  +\sum_jt_jP_j\bigl[\Gamma v_1-v_0+(\Gamma-1)C_Vp_j\bigr].
  \label{eq:sin-majority-numerator}
\end{align}
Then $\mathsf D(t)>0$ on the cube and $A(w)=\mathsf N(w)/\mathsf D(w)$ on binary inputs.

\subsection{The necessary positional compensation}
\label{app:sinusoidal-majority-compensation}

Let $d_i=(W_2)_{1i}-(W_2)_{0i}$ and $b=b_{2,1}-b_{2,0}$. After multiplying the logit difference by the positive denominator $\mathsf D(t)$, the resulting function is continuous and piecewise affine in $t$. For an activation pattern $\sigma\subseteq[h]$, define
\begin{align}
  \ell_\sigma
  &=\sum_{i\in\sigma}d_i(W_1)_{i,:}^\top,
  &
  c_\sigma
  &=b+\sum_{i\in\sigma}d_ib_{1,i},
  \notag\\
  \delta_\sigma
  &=\ell_\sigma^\top(v_1-v_0),
  &
  \xi_\sigma
  &=C_V^\top\ell_\sigma,
  \label{eq:sin-majority-piece-parameters}
\end{align}
and
\begin{equation}
  c_\sigma^*
  =
  -\ell_\sigma^\top\frac{v_0+\Gamma v_1}{1+\Gamma}
  -\xi_\sigma^\top\mu.
  \label{eq:sin-majority-bias-center}
\end{equation}

\begin{lemma}[Positional compensation]
\label[lemma]{lem:sin-majority-compensation}
Outside a set of raw-cube measure zero, every exact solution at a sufficiently large odd length has an activation pattern $\sigma\subseteq[h]$ such that $\delta_\sigma\ne0$ and
\begin{equation}
  \left\|
    a+\frac{\sinh\lambda}{\delta_\sigma}\xi_\sigma
  \right\|_2
  \le\frac{C}{n},
  \label{eq:sin-majority-key-constraints}
\end{equation}
while
\begin{equation}
  |c_\sigma-c_\sigma^*|
  \le
  \frac{C|\delta_\sigma|}{n}
  \frac{\Gamma}{(1+\Gamma)^2}.
  \label{eq:sin-majority-bias-constraint}
\end{equation}
The constants are uniform over the bounded raw parameter cube, including when $q$ and $\delta_\sigma$ are arbitrarily small but nonzero.
\end{lemma}

\begin{proof}
Exclude first the finite union of sets on which one of the affine pieces defined below vanishes at $t=0$; each is null by slicing along $b_{2,1}$. On activation pattern $\sigma$, the denominator-scaled logit difference is an affine function $L_\sigma(t)=L_\sigma(0)+\sum_j\widehat\ell_{\sigma,j}t_j$, where
\begin{equation}
  \widehat\ell_{\sigma,j}
  =P_j(\kappa_\sigma+\nu_\sigma^\top p_j),
  \qquad
  \kappa_\sigma
  =\ell_\sigma^\top(\Gamma v_1-v_0)+(\Gamma-1)c_\sigma,
  \qquad
  \nu_\sigma=(\Gamma-1)\xi_\sigma.
  \label{eq:sin-majority-affine-coefficients}
\end{equation}
Fix $t\in[0,1]^n$ with $\sum_jt_j=n/2$. Preserving $\sum_jt_j$, $\sum_jP_jt_j$, $\sum_jP_j\sin j\,t_j$, and $\sum_jP_j\cos j\,t_j$, one can move to a point with at most four fractional coordinates. Rounding those coordinates to words with respectively $m$ and $m+1$ ones and using exact correctness forces a zero of one affine piece within four coordinate changes. Applying this to the centrally symmetric convex body generated by $P_j(1,p_j)\in\mathbb R^3$ and using the elementary convex-body slab bound yields, for some $\sigma$,
\begin{equation}
  \sum_{j=1}^n
  \left|
    \widehat\ell_{\sigma,j}-\mathrm{med}(\widehat\ell_{\sigma,1},\ldots,\widehat\ell_{\sigma,n})
  \right|
  \le C_hM_\sigma,
  \qquad
  \left|L_\sigma\!\left(\tfrac12\mathbf1\right)\right|
  \le C_hM_\sigma,
  \label{eq:sin-majority-flat-piece}
\end{equation}
where $M_\sigma=\max_j|\widehat\ell_{\sigma,j}|$. The fixed phases $j\bmod2\pi$ are equidistributed because $|\sum_{j=1}^ne^{ikj}|=O_k(1)$ for each fixed nonzero integer $k$. Applied to the sequence $e^{a^\top p_j}(\kappa_\sigma+\nu_\sigma^\top p_j)$, the first inequality in \eqref{eq:sin-majority-flat-piece} first bounds $\|a\|_2$, then a compactness argument forces $a$ and $\nu_\sigma/\kappa_\sigma$ to approach zero, and finally the first two discrete Fourier modes of the second-order Taylor expansion give
\begin{equation}
  \kappa_\sigma\ne0,
  \qquad
  M_\sigma\asymp|\kappa_\sigma|,
  \qquad
  \|a\|_2^2\le\frac{C}{n},
  \qquad
  \left\|a+\frac{\nu_\sigma}{\kappa_\sigma}\right\|_2\le\frac{C}{n}.
  \label{eq:sin-majority-fourier-conclusion}
\end{equation}
To return to the raw parameters, put $e_c=c_\sigma-c_\sigma^*$ and $\tau=2\Gamma\delta_\sigma/(1+\Gamma)$. Direct substitution gives
\begin{equation}
  L_\sigma\!\left(\tfrac12\mathbf1\right)
  =\frac{(1+\Gamma)T}{2}e_c,
  \qquad
  \kappa_\sigma
  =\tau-\nu_\sigma^\top\mu+(\Gamma-1)e_c.
  \label{eq:sin-majority-center-identities}
\end{equation}
Since $T\asymp n$, \eqref{eq:sin-majority-flat-piece} and \eqref{eq:sin-majority-fourier-conclusion} imply $|e_c|\le C|\kappa_\sigma|/((1+\Gamma)n)$ and $|\kappa_\sigma-\tau|\le Cn^{-1/2}|\kappa_\sigma|$. Hence $\delta_\sigma\ne0$, $|\tau|\asymp|\kappa_\sigma|$, and $\|\nu_\sigma/\tau-\nu_\sigma/\kappa_\sigma\|_2\le C/n$. Using $\nu_\sigma/\tau=(\sinh\lambda/\delta_\sigma)\xi_\sigma$ proves \eqref{eq:sin-majority-key-constraints}, and the bound on $e_c$ proves \eqref{eq:sin-majority-bias-constraint}.
\end{proof}

\subsection{Volume integration}
\label{app:sinusoidal-majority-volume-proof}

\begin{proof}[Proof of \cref{thm:sinusoidal-majority-volume}]
For the upper bound, fix one of the finitely many activation patterns $\sigma$ and condition on $q$, the two symbol key columns, and all value and classifier coordinates except $b_{2,1}$, leaving the two positional key columns unconditioned. Each coordinate of \eqref{eq:sin-majority-key-constraints} restricts one independent positional key column to a slab whose conditional normalized width is at most
\begin{equation}
  C_B\min\left\{1,\frac{1}{n\beta_n\|q\|_\infty}\right\}.
  \label{eq:sin-majority-key-slab-width}
\end{equation}
The two columns use disjoint raw coordinates. After they are fixed, \eqref{eq:sin-majority-bias-constraint} restricts $b_{2,1}$ to an interval of normalized length $O_{B,h}(1/n)$. Summing over $\sigma$ and integrating $t=\|q\|_\infty$, whose density under the uniform square measure is proportional to $t$, gives
\begin{align}
  V_B^{\MAJORITY,\mathcal A_{\mathrm{sin,gen}}(h)}(n)
  &\le
  \frac{C_{B,h}}{n}
  \int_0^B
  t\min\left\{1,\frac{1}{n\beta_nt}\right\}^2dt
  \notag\\
  &\le
  \frac{C_{B,h}(1+\log(n\beta_n))}{n^3\beta_n^2}
  =O_{B,h}\!\left(\frac{1}{n^3\log n}\right).
  \label{eq:sin-majority-upper-integral}
\end{align}

For the lower bound, restrict the value and classifier coordinates to a fixed positive-volume open box on which the bias-free logit difference has the form $\ell^\top A+c_H$ for every attainable attention output, with constants $0<\delta_-\le\delta:=\ell^\top(v_1-v_0)\le\delta_+$ and $\xi=C_V^\top\ell$ satisfying $\|\xi\|_2/\delta\le1/100$. Such a box is obtained from an interior center with $v_0=0$, $v_1=\eta(1,0)^\top$, $C_V=0$, and one uniformly active hidden unit. Set
\begin{equation}
  A^*
  =\frac{v_0+\Gamma v_1}{1+\Gamma}+C_V\mu,
  \qquad
  b_{2,1}^*
  =b_{2,0}-c_H-\ell^\top A^*.
  \label{eq:sin-majority-lower-center}
\end{equation}
At this centered bias, the exact logit difference is
\begin{equation}
  (y(w))_1-(y(w))_0
  =
  \frac{\sum_j(w_j-1/2)\widehat\ell_j}{\mathsf D(w)},
  \qquad
  \widehat\ell_j
  =P_j\left[\frac{2\Gamma\delta}{1+\Gamma}+(\Gamma-1)\xi^\top(p_j-\mu)\right].
  \label{eq:sin-majority-lower-identity}
\end{equation}
If $|\lambda|\le c_1n^{-1/2}$ and $\|a+(\sinh\lambda/\delta)\xi\|_2\le c_2/n$, a uniform first-order expansion of \eqref{eq:sin-majority-lower-identity} gives $\widehat\ell_j/\tau=F_0+O(1/n)$ with $\tau=2\Gamma\delta/(1+\Gamma)$ and $F_0\ge3/4$. Therefore the numerator has the \MAJORITY\ sign and magnitude bounded below by a positive constant, while $\mathsf D(w)=O(n)$; an interval $|b_{2,1}-b_{2,1}^*|<c_3/n$ preserves strict correctness on every input.

It remains to measure a full-dimensional set satisfying these conditions. Restrict the query bias to
\begin{equation}
  q_1>0,
  \qquad
  |q_2|\le q_1/4,
  \qquad
  r:=\beta_nq_1\in\left[\frac{L}{n},\frac{c}{\sqrt n}\right].
  \label{eq:sin-majority-query-cone}
\end{equation}
For sufficiently small $c$, the symbol key columns may vary over their full raw squares while $|\lambda|\le c_1n^{-1/2}$. Conditional on these choices and on the decoder box, the two coordinates of the compensation condition restrict one entry of each positional key column to an interval of length $\Theta(1/(nr))$; fixed-width restrictions on the other entries keep these intervals inside the raw cube. Their joint normalized mass is therefore $\Omega_{B,h}(1/(n^2r^2))$, and the admissible output-bias interval contributes $\Omega_{B,h}(1/n)$. Since the query-cone measure after integrating $q_2$ is a constant multiple of $r\,dr/\beta_n^2$,
\begin{align}
  V_B^{\MAJORITY,\mathcal A_{\mathrm{sin,gen}}(h)}(n)
  &\ge
  \frac{c_{B,h}}{n^3\beta_n^2}
  \int_{L/n}^{c/\sqrt n}\frac{dr}{r}
  \notag\\
  &\ge
  \frac{c'_{B,h}}{n^3\log n}.
  \label{eq:sin-majority-lower-integral}
\end{align}
Every restricted raw coordinate varies over an interval of positive width and every unrestricted coordinate integrates out, so this is a full-dimensional lower bound. Together with \eqref{eq:sin-majority-upper-integral}, it proves the theorem.
\end{proof}

\section{Experimental Details and Selected Additional Results}
\label{app:experiments}

This appendix specifies the experimental protocol used for the main-text comparisons and reports the additional results referenced in the discussion of sinusoidal positional encoding. 

\subsection{Reported models and benchmark provenance}
\label{app:experiment-scope}

The baseline positional implementation is the model $\Apos(h)$ from \cref{sec: setup}, with quadratic positional coordinates $(j,j^2)$, and the selectively normalized implementation is $\Anorm(h)$ from \cref{sec:index-norm}. For the baseline \MAJORITY\ experiment, we use the unrestricted affine NoPE implementation of $\Acount(h)$. The additional sinusoidal \MAJORITY\ control is the fixed-frequency general model $\mathcal A_{\mathrm{sin}}(h)$ defined in \cref{app:sinusoidal-majority-model}. 

We use the same $20$ seeds, $2000$--$2019$, at each maximum training length
\begin{equation}
  N\in\{11,21,31,41,51\}.
  \label{eq:experimental-training-maxima}
\end{equation}
Training and qualification did not have access to the test banks, and every newly trained checkpoint was selected and frozen before its test evaluation.

\subsection{Training, selection, and initialization}
\label{app:training-protocol}

Training uses Adam with learning rate $3\times10^{-4}$, $\beta_1=0.9$, $\beta_2=0.999$, numerical stabilizer $10^{-8}$, no weight decay, batch size $256$, and gradient-norm clipping at one. Each update uses a fresh batch of independent uniformly random bits. For \INDEX, the index $k$ is sampled uniformly from the positions in the sampled sequence.  For \FIRST\ and \INDEX, one data length is sampled uniformly from $\{1,\ldots,N\}$ per batch.  For \MAJORITY, the sampled length is uniform over the odd integers in that range. Computation uses \texttt{float64}.  At data length $L$, attention scores are multiplied by $\log L$, without an additional $1/\sqrt d$ factor.

Every run is trained for at most $30{,}000$ updates and is validated every $500$ updates on a fixed bank of $1{,}000$ examples at length $N$. Training stops after three consecutive perfect validations. The selected checkpoint maximizes validation accuracy; ties are resolved first by lower cross-entropy and then by the earlier update. No input longer than $N$ is used in checkpoint selection. 

For the $\Apos(h)$ and $\Anorm(h)$ models, the vectors $a,r,s,b,c_K$ are initialized from centered normal distributions with standard deviations $0.1/N$, $0.1/N$, $0.1/N^2$, $0.1$, $0.1$,  \label{eq:positional-initialization-scales} respectively.  Other parameters use Xavier-uniform initialization and their biases are zero.  $\Apos(h)$/$\Anorm(h)$ pairs use identical initial tensors and the same training streams. For $\Acount$ and $\Asin$, content query/key blocks and the bounded sinusoidal positional blocks use Xavier-uniform initialization, while shared value, bias, and classifier blocks are copied from the paired $\Apos$ seed.

\subsection{Evaluation banks, metrics, and uncertainty}
\label{app:evaluation-banks}

For a maximum training length $N$, the common test grid is
\begin{equation}
  N,
  \quad 2N+1,
  \quad 3N,
  \quad 5N,
  \quad 8N+1,
  \quad 10N+1.
  \label{eq:experimental-test-grid}
\end{equation}
At each test length, all seeds and model variants for the same task use the same saved bank of $1{,}000$ examples.  For \FIRST, the first bit is exactly balanced and the remaining bits are random.  For \MAJORITY, $500$ random strings are paired with their complements. For \INDEX, each of $500$ random string--query pairs is paired with the complemented string and the same query.  

Test accuracy is the fraction correct on one saved bank.  The longest-test statistic is the mean across the $20$ seeds at $L=10N+1$. For coefficient analyses and other OOD summaries, the seed-level mean OOD accuracy is the equal-weight mean across the five test lengths strictly greater than $N$.  Confidence intervals are pointwise percentile $95\%$ intervals. 

\subsection{Coefficient analysis}
\label{app:index-coefficient-experiments}

The coefficient analysis uses the $100$ selected \INDEX\ checkpoints on $\Apos$: $20$ seeds at each of the five values of $N$ in \eqref{eq:experimental-training-maxima}.  For each checkpoint, we compute
\begin{equation}
  \chi=a^\top s,
  \qquad
  \ell=2b^\top s+a^\top r
\end{equation}
directly from the learned parameters.

To avoid treating differences among the five training maxima as evidence for a coefficient association, we rank each coefficient magnitude and the seed-level mean OOD accuracy separately within each value of $N$, and then pool the within-$N$ ranks.  The resulting correlations are $-0.655$ for $N^2|\chi|$ and $-0.516$ for $N|\ell|$.  The two-sided permutation test uses $10{,}000$ permutations and shuffles ranks only within the same value of $N$; both tests give $p<10^{-3}$. 

\subsection{Selected results across maximum training lengths}
\label{app:additional-training-lengths}
\label{app:index-norm-experiments}

\Cref{tab:selected-longest-test-results} reports the longest-test results.  Each entry is the mean percentage accuracy across $20$ seeds, followed by its pointwise $95\%$ seed-bootstrap interval. 

\begin{table}[t]
  \centering
  \caption{Accuracy at the longest tested length $L=10N+1$.
  Entries are percentages, with pointwise $95\%$ seed-bootstrap
  intervals in parentheses.}
  \label{tab:selected-longest-test-results}
  \scriptsize
  \resizebox{\linewidth}{!}{%
  \begin{tabular}{ccrrrr}
    \toprule
    $N$ & $L$
      & \FIRST\ ($\Apos$)
      & \MAJORITY\ ($\Acount$)& \INDEX\ ($\Apos$)
      & \INDEX\ ($\Anorm$) \\
    \midrule
    11 & 111 & 100.0 (100.0--100.0) & 83.8 (78.0--89.4) & 57.2 (55.9--58.6) & 63.4 (62.7--64.1) \\
    21 & 211 & 100.0 (100.0--100.0) & 91.3 (86.2--95.3) & 55.2 (54.1--56.5) & 74.1 (68.8--79.9) \\
    31 & 311 & 100.0 (100.0--100.0) & 87.1 (82.6--91.5) & 59.2 (58.3--60.0) & 90.6 (85.6--95.4) \\
    41 & 411 & 100.0 (100.0--100.0) & 88.5 (84.2--92.2) & 55.1 (54.5--55.8) & 84.4 (78.1--90.6) \\
    51 & 511 & 100.0 (100.0--100.0) & 82.6 (77.6--87.5) & 59.6 (58.6--60.6) & 85.1 (80.2--90.1) \\
    \bottomrule
  \end{tabular}%
  }
\end{table}

\subsection{Sinusoidal positional encoding on \MAJORITY}
\label{app:sinusoidal-majority-experiments}

Relevant results on the performance of $\Asin$ compared to $\Acount$ on \MAJORITY{} is displayed in \cref{tab:sinusoidal-majority}.
\begin{table}[t]
  \centering
  \caption{Mean OOD accuracy on \MAJORITY\ for $\Acount$ and $\Asin$
  general models.  Entries are percentages, averaged equally over the
  five test lengths above $N$, with $95\%$ confidence
  intervals in parentheses.}
  \label{tab:sinusoidal-majority}
  \begin{tabular}{crr}
    \toprule
    $N$ & NoPE General & Sinusoidal General \\
    \midrule
    11 & 91.6 (88.1--95.0) & 95.3 (93.6--96.8) \\
    21 & 95.1 (91.8--97.6) & 93.8 (91.5--95.9) \\
    31 & 92.9 (90.0--95.6) & 91.2 (88.2--94.0) \\
    41 & 93.4 (90.6--95.8) & 91.5 (88.7--94.1) \\
    51 & 89.8 (86.2--92.9) & 92.9 (90.0--95.6) \\
    \bottomrule
  \end{tabular}
\end{table}

\subsection{Sanity checks for $\Anorm(h)$ on other tasks}
\label{app:norm-sanity-checks}

We also compared $\Apos(h)$ to $\Anorm(h)$  on \FIRST\ and \MAJORITY.  For \FIRST, both positional models achieve $100\%$ mean accuracy at $L=10N+1$ for every $N$ in \eqref{eq:experimental-training-maxima}. 

For our setup of \MAJORITY, $\Apos$ and $\Anorm$ are equivalent, hence producing same results for every seed.

\end{document}